\documentclass[12pt]{article} 

\usepackage{amsmath,amsthm,amssymb,amsfonts}
\usepackage{algorithm}
\usepackage{algorithmic}
\usepackage{graphicx}
\usepackage{indentfirst}
\usepackage{caption}

\makeatletter
\newcommand{\figcaption}[1]{\def\@captype{figure}\caption{#1}}
\newcommand{\tblcaption}[1]{\def\@captype{table}\caption{#1}}
\def\Hline{%
\noalign{\ifnum0=`}\fi\hrule \@height 2pt \futurelet
\reserved@a\@xhline}
\renewcommand{\@maketitle}{\newpage
\begin{center}
{\LARGE \@title \par} \vskip 1.5em {\large \lineskip .5em
\begin{tabular}[t]{c}\@author
\end{tabular}\par}
{\large \@date} 
\end{center}
\par}
\renewcommand\section{\@startsection {section}{1}%
{\z@}%
{-3ex}%
{2ex}%
{\normalfont\normalsize\bfseries}}
\renewcommand\subsection{\@startsection {subsection}{1}%
{\z@}%
{-2ex}%
{0.5ex}%
{\normalfont\normalsize\itshape}}
\renewcommand\subsubsection{\@startsection {subsubsection}{1}%
{\z@}%
{-2ex}%
{0.5ex}%
{\normalfont\normalsize\itshape}}
\makeatother

\newtheorem{theo}{Theorem}[section]

\title{{\Large Modular Representation of Layered Neural Networks}}

\author{{\normalsize Chihiro Watanabe}\thanks{\textit{Email address:} watanabe.chihiro@lab.ntt.co.jp}{\normalsize ,\ \ Kaoru Hiramatsu,\ \ Kunio Kashino}\\
{\small\itshape NTT Communication Science Laboratories,}\\{\small\itshape 3-1, Morinosato Wakamiya, Atsugi-shi, Kanagawa Pref. 243-0198 Japan}}
\date{}

\begin{document}

\maketitle

\begin{abstract}
{\normalsize Layered neural networks have greatly improved the performance of various applications including image processing, speech recognition, natural language processing, and bioinformatics. However, it is still difficult to discover or interpret knowledge from the inference provided by a layered neural network, since its internal representation has many nonlinear and complex parameters embedded in hierarchical layers. Therefore, it becomes important to establish a new methodology by which layered neural networks can be understood. 

In this paper, we propose a new method for extracting a global and simplified structure from a layered neural network. Based on network analysis, the proposed method detects communities or clusters of units with similar connection patterns. We show its effectiveness by applying it to three use cases. 
(1) Network decomposition: it can decompose a trained neural network into multiple small independent networks thus dividing the problem and reducing the computation time. (2) Training assessment: the appropriateness of a trained result with a given hyperparameter or randomly chosen initial parameters can be evaluated by using a modularity index. And (3) data analysis: in practical data it reveals the community structure in the input, hidden, and output layers, which serves as a clue for discovering knowledge from a trained neural network. \vspace{0.5em}}
\end{abstract}

\newpage

\section{Introduction} 
\label{sec:intro}

Layered neural networks have recently been applied to various tasks \cite{Bengio2013, LeCun2015}, including image processing \cite{Krizhevsky2012,Tompson2014}, speech recognition \cite{Hinton2012,Sainath2013}, natural language processing \cite{Collobert2011,Sutskever2014}, and bioinformatics \cite{Leung2014, Xiong2015}. Although they have simple layered structures of units and connections, they outperform other conventional models by their ability to learn complex nonlinear relationships between input and output data. In each layer, inputs are transformed into more abstract representations under a given set of the model parameters. These parameters are automatically optimized through training so that they extract the important features of the input data. In other words, it does not require either careful feature engineering by hand, or expert knowledge of the data. This advantage has made layered neural networks successful in a wide range of tasks, as mentioned above.

However, the inference provided by a layered neural network consists of a large number of nonlinear and complex parameters, which makes it difficult for human beings to understand it. 
More complex relationships between input and output can be represented as the network becomes deeper or the number of units in each hidden layer increases, however interpretation becomes more difficult. The large number of parameters also causes problems in terms of computational time, memory and over-fitting, so it is important to reduce the parameters appropriately. 
Since it is difficult to read the underlying structure of a neural network and to identify the parameters that are important to keep, we must perform experimental trials to find the appropriate values of the hyperparameters and the random initial parameters that achieve the best trained result. 

In this paper, to overcome such difficulties, we propose a new method for extracting a global and simplified structure from a layered neural network (For example, Figure \ref{fig:exp1m} and \ref{fig:exp3m}). Based on network analysis, the proposed method defines a modular representation of the original trained neural network by detecting communities or clusters of units with similar connection patterns. 
Although the modular neural network proposed by \cite{Jacobs1991, Azam2000} has a similar name, it takes the opposite approach to ours. In fact, it constructs the model structure before training with multiple split neural networks inside it. Then, each small neural network works as an expert of a subset task. Our proposed method is based on the community detection algorithm. To date, various methods have been proposed to express the characteristics of diverse complex networks without layered structures \cite{Newman2004, Estrada2005, Newman2006, Meunier2006, Newman2007}, however, no method has been developed for detecting the community structures of trained layered neural networks. 

The difficulty of conventional community detection from a layered neural network arises from the fact that an assumption commonly used in almost all conventional methods does not hold for layered neural networks: to detect the community structure of network, previous approaches assume that there are more intra-community edges that connect vertices inside a community than inter-community edges that connect vertices in mutually different communities. A network with such a characteristic is called assortative. This seems to be a natural assumption, for instance, for a network of relationships between friends. In layered neural networks, however, units in the same layer do not connect to each other and they only connect via units in their parent or child layers. This characteristic is similar to that of a bipartite graph, and such networks are called disassortative. It is not appropriate to apply conventional methods based on the assumption of an assortative network to a layered neural network. A basic community detection method that can be applied to either assortative or disassortative networks has been proposed by Newman et al \cite{Newman2007}. In this paper, we propose an extension of this method for extracting modular representations of layered neural networks.

The proposed method can be employed for various purposes. In this paper, we show its effectiveness with the following three applications. 
\begin{enumerate}
\item \textbf{Decomposition of layered neural network into independent networks}: the proposed method decomposes a trained neural network into multiple small independent neural networks. In such a case, the output estimation by the original neural network can be regarded as a set of independent estimations made by the decomposed neural networks. In other words, it divides the problem and reduces the overall computation time. In section \ref{sec:decomposition}, we show that our method can properly decompose a neural network into multiple independent networks, where the data consist of multiple independent vectors. 
\item \textbf{Generalization error estimation from community structure}: modularity \cite{Newman2004} is defined as a measure of the effectiveness of a community detection result. Section \ref{sec:gee} reveals that there is a correlation between modularity and the generalization error of a layered neural network. It is shown that the appropriateness of the trained result can be estimated from the community structure of the network. 
\item \textbf{Knowledge discovery from modular representation}: the modular representation extracted by the proposed method serves as a clue for understanding the trained result of a layered neural network. 
It extracts the community structure in the input, hidden, and output layer. In section \ref{sec:kd}, we introduce the result of applying the proposed method to practical data. 
\end{enumerate}

The remaining part of this paper is composed as follows: we first describe a layered neural network model in section \ref{sec:lnn}. Then, we explain our proposed method for extracting a modular representation of a neural network in section \ref{sec:communitydetect}. The experimental results are reported in section \ref{sec:experiment}, which show the effectiveness of the proposed method in the above three applications. In section \ref{sec:discussion}, we discuss the experimental results. Section \ref{sec:conclusion} concludes this paper. 

\section{Layered neural networks} 
\label{sec:lnn}

We start by defining $x\in \mathbb{R}^M,\ y\in \mathbb{R}^N$ and a probability density function $q(x,y)$ on $\mathbb{R}^M\times \mathbb{R}^N$. A training data set $\{(X_i,Y_i)\}_{i=1}^n$ with a sample size $n$ is assumed to be generated independently from $q(x,y)$. Let $f(x,w)$ be a function from $x\in \mathbb{R}^M,\ w\in \mathbb{R}^L$ to $\mathbb{R}^N$ of a layered neural network that estimates an output $y$ from an input $x$ and a parameter $w$. 

For a layered neural network, $w=\{\omega^d_{ij}, \theta^d_i\}$, where $\omega^d_{ij}$ is the weight of connection between the $i$-th unit in the depth $d$ layer and the $j$-th unit in the depth $d+1$ layer, and $\theta^d_i$ is the bias of the $i$-th unit in the depth $d$ layer. A layered neural network with $D$ layers is represented by the following function: 
\begin{eqnarray*}
  f_j(x,w) &=& \sigma(\sum_i \omega^{D-1}_{ij} o^{D-1}_i+\theta^{D-1}_j),\\
  o^{D-1}_j &=& \sigma(\sum_i \omega^{D-2}_{ij} o^{D-2}_i+\theta^{D-2}_j),\\
  &\vdots& \\
  o^2_j &=& \sigma(\sum_i \omega^1_{ij} x_i +\theta^1_j),
\end{eqnarray*}
where a sigmoid function is defined by
\begin{eqnarray*}
  \sigma (x)=\frac{1}{1+\exp (-x)}.
\end{eqnarray*}

The training error $E(w)$ and the generalization error $G(w)$ are respectively defined by
\begin{eqnarray*}
  E(w) &=& \frac{1}{n} \sum_{i=1}^n \|Y_i-f(X_i,w)\|^2,\\
  G(w) &=& \int \|y-f(x,w)\|^2 q(x,y)dxdy,
\end{eqnarray*}
where $\|\cdot\|$ is the Euclidean norm of $\mathbb{R}^N$.

The generalization error is approximated by
\begin{eqnarray*}
  G(w)\approx \frac{1}{m} \sum_{j=1}^m \|{Y_j}'-f({X_j}',w)\|^2,
\end{eqnarray*}
where $\{({X_j}', {Y_j}')\}_{j=1}^m$ is a test data set that is independent of the training data set.

To construct a sparse neural network, we adopt the LASSO method \cite{Ishikawa1990,Tibshirani1994} in which the minimized function is defined by
\begin{eqnarray*}
  H(w) &=& \frac{n}{2}\ E(w)+\lambda \sum_{d,i,j} |\omega^d_{ij}|,
\end{eqnarray*}
where $\lambda$ is a hyperparameter.

The parameters are trained by the stochastic steepest descent method, 
\begin{eqnarray}
  \Delta w &=& -\eta \nabla H_i (w)\nonumber \\
  &=& -\eta \ \Bigl(\frac{1}{2}\nabla \{\|Y_i-f(X_i,w)\|^2\}+\lambda \ \mathrm{sgn}(w)\Bigr),
  \label{eq:dw}
\end{eqnarray}
where $H_i (w)$ is the training error computed only from the $i$-th sample $(X_i, Y_i)$. 
Here, $\eta$ is defined for training time $t$ such that
\begin{eqnarray*}
  \eta(t)\propto \frac{1}{t},
\end{eqnarray*}
which is sufficient for convergence of the stochastic steepest descent. 
Eq. (\ref{eq:dw}) is numerically calculated by the following procedure, which is called error back propagation \cite{Werbos1974,Rumelhart1986}: for the $D$-th layer, 
\begin{eqnarray*}
  \delta^{D}_j &=& (o^{D}_j-y_j)\ o^{D}_j\ (1-o^{D}_j),\\
  \Delta \omega^{D-1}_{ij} &=& -\eta (\delta^{D}_j o^{D-1}_i+\lambda \ \mathrm{sgn}(\omega^{D-1}_{ij})),\\
  \Delta \theta^D_j &=& -\eta \delta^{D}_j.
\end{eqnarray*}
For $d=D-1,\ D-2,\cdots, 2$,
\begin{eqnarray*}
  \delta^d_j &=& \sum_{k=1}^{l_{d+1}} \delta^{d+1}_k \omega^d_{jk}\ o^d_j\ (1-o^d_j),\\
  \Delta \omega^{d-1}_{ij} &=& -\eta (\delta^d_j o^{d-1}_i+\lambda \ \mathrm{sgn}(\omega^{d-1}_{ij})),\\
  \Delta \theta^d_j &=& -\eta \delta^d_j.
\end{eqnarray*}

Algorithm \ref{alg_BP} is used for training a layered neural network based on error back propagation. 
With this algorithm, we obtain a neural network whose redundant weight parameters are close to zero. 

\begin{algorithm}
\caption{Stochastic steepest descent algorithm of a layered neural network}
\label{alg_BP}
\begin{algorithmic}
\FOR{$i=1$ to $a_1*n$}
	\STATE Randomly sample $k$ from uniform distribution on $\{1, 2, \cdots, n\}$.
	\STATE $x_j \gets x^k_j$,
	\STATE $y_j \gets y^k_j$,
	\STATE where $x^k_j$ and $y^k_j$ is the $j$-th element of $k$-th sample.
	\STATE $\eta =0.8 \times \frac{a_1\times n}{a_1\times n+5\times i}$,
	\STATE where $a_1\times n$ is the number of iterations. (Here, we defined $\eta$ so that it gets smaller and smaller as the iteration proceeds, to accelerate convergence of the algorithm.)
	\STATE (1) Output calculation of all layers: let $o^d_j$ be an output of the $j$-th unit in the depth $d$ layer.
	\STATE $o^1_j \gets x_j$.
	\FOR{$d=2$ to $D$}
		\STATE $o^d_j \gets \sigma (\sum_i \omega^{d-1}_{ij} o^{d-1}_i +\theta^d_j)$.
	\ENDFOR
	\STATE (2) Update weight $\omega^d_{ij}$ and bias $\theta^d_j$ based on back propagation, where $\epsilon>0$ is a small constant.
	\STATE $\delta^{D}_j \gets (o^{D}_j-y_j)(o^{D}_j(1-o^{D}_j)+\epsilon)$.
	\STATE $\Delta \omega^{D-1}_{ij} \gets -\eta (\delta^{D}_j o^{D-1}_i+\lambda \ \mathrm{sgn}(\omega^{D-1}_{ij}))$. 
	\STATE $\Delta \theta^{D}_j \gets -\eta \delta^{D}_j$.
	\FOR{$d=D-1$ to $2$}
		\STATE $\delta^d_j \gets \sum_{j'} \delta^{d+1}_{j'} \omega^d_{jj'} (o^d_j (1-o^d_j)+\epsilon)$.
		\STATE $\Delta \omega^{d-1}_{ij} \gets -\eta (\delta^d_j o^{d-1}_i+\lambda \ \mathrm{sgn}(\omega^{d-1}_{ij}))$.
	\STATE $\Delta \theta^d_j \gets -\eta \delta^d_j$.
	\ENDFOR
	\STATE $\omega^d_{ij}\gets \omega^d_{ij}+\Delta \omega^d_{ij}$.
	\STATE $\theta^d_j \gets \theta^d_j+\Delta \theta^d_j$.
\ENDFOR
\end{algorithmic}
\end{algorithm}

\section{Modular representation of layered neural networks} 
\label{sec:communitydetect}

\begin{figure*}
  \centering
  \includegraphics[width=130mm]{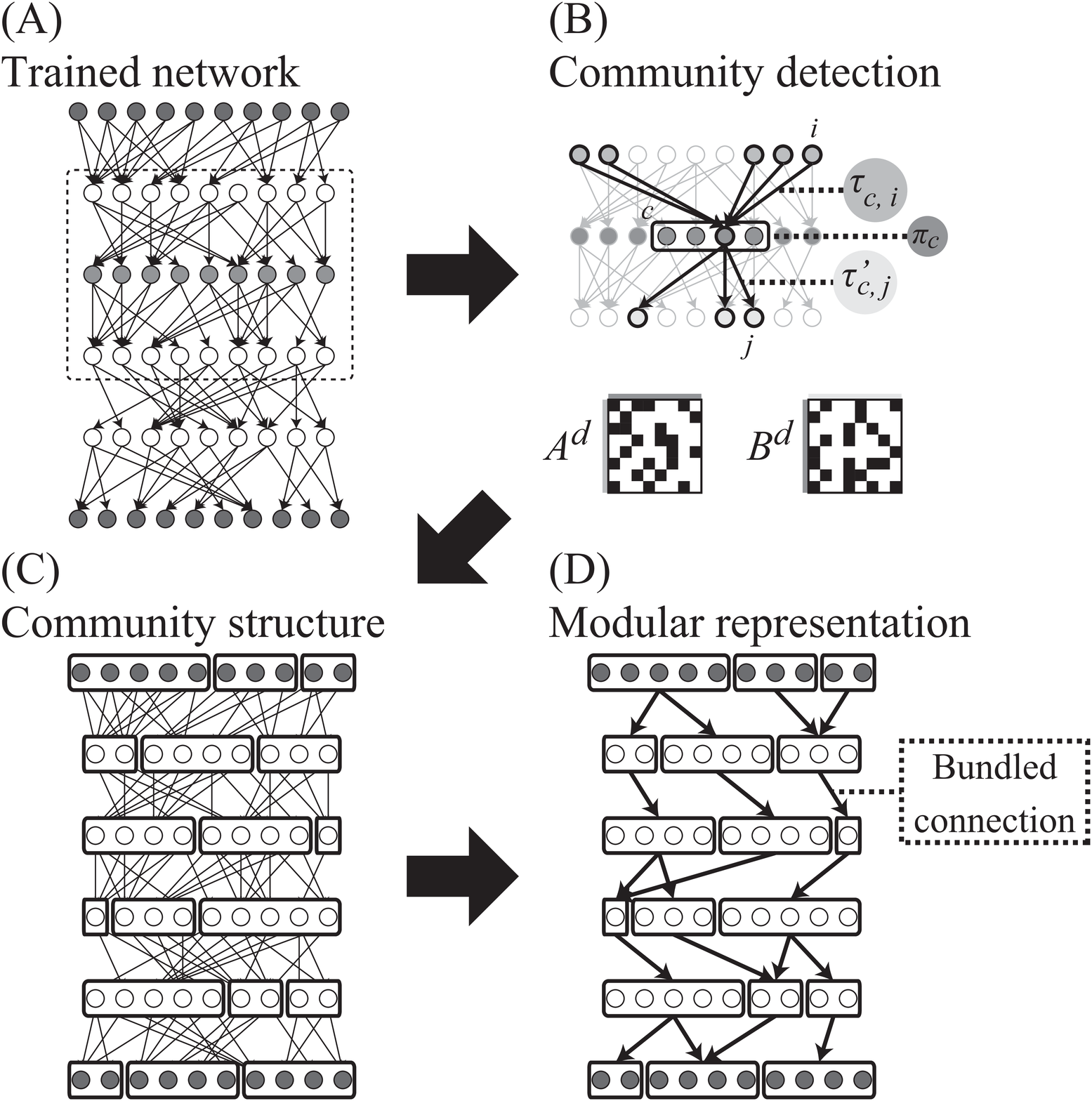}
  \caption{Proposed method. 
(A) Trained network: a layered neural network is trained by the stochastic steepest descent method. 
(B) Community detection: the connections between every layer and its adjacent layers are represented by partial network matrices $A^d$ and $B^d$. communities in each layer are extracted by using network analysis. 
(C) Community structure: the community assignments of all units are determined from the estimated parameters in (B). 
(D) Modular representation: bundled connections are defined that summarize multiple connections between pairs of communities. 
}
  \label{fig:com}
\end{figure*}

Here we propose a new community detection method, which is applied to any layered neural networks (Figure \ref{fig:com} (A)). The proposed method is an extension of the basic approach proposed by Newman et al \cite{Newman2007}. It detects communities of  assortative or disassortative networks. The key idea behind our method is that the community assignment of the units in each layer is estimated by using connection with adjacent layers. 

As shown in Figure \ref{fig:com} (B), a partial network consisting of the connections between every layer and its adjacent layers is represented in the form of two matrices: $A^d=\{A^d_{ij}\}$ and $B^d=\{B^d_{ij}\}$. The matrix $A^d$ and $B^d$ represent the connections between two layers of depth $d-1$ and $d$, and two layers of depth $d$ and $d+1$, respectively. 
In this paper, an element $A^d_{ij}$ is given by
\begin{eqnarray}
  A^d_{ij} = \begin{cases}
    1 & (|\omega^{d-1}_{ij}|\geq \xi), \\
    0 & (otherwise),
  \end{cases}  
  \label{eq:Aelement}
\end{eqnarray}
where $\xi$ is called a weight removing hyperparameter. In a similar way, an element $B^d_{ij}$ is given by
\begin{eqnarray}
  B^d_{ij} = \begin{cases}
    1 & (|\omega^{d}_{ij}|\geq \xi), \\
    0 & (otherwise).
  \end{cases}  
  \label{eq:Belement}
\end{eqnarray}
For simplicity, we denote $A^d$ and $B^d$ as $A$ and $B$, respectively, in the following explanation.

Our method is based on the assumption that units in the same community have a similar probability of connection from/to other units. 
This assumption is almost the same as that in the previous method \cite{Newman2007}, except that our method utilizes both incoming and outgoing connections of each community, and it detects communities in individual layers. 
Therefore, the community detection result is derived in a similar way to the previous method \cite{Newman2007}, as explained in the rest of this section. 
As shown on the right in Figure \ref{fig:com} (B), the statistical model for community detection has three kinds of parameters. The first parameter $\pi=\{\pi_c\}$ represents the prior probability of a unit in the depth $d$ layer that belongs to the community $c$. The conditional probability of connections for a given community $c$ is represented by the second and third parameters $\tau=\{\tau_{c,i}\}$ and $\tau'=\{\tau'_{c,j}\}$, where $\tau_{c,i}$ represents the probability that a connection to a unit in the community $c$ is attached from the $i$-th unit in the depth $d-1$ layer. Similarly, $\tau'_{c,j}$ represents the probability that a connection from a unit in the community $c$ is attached to the $j$-th unit in the depth $d+1$ layer. 
Here, we omit the index $d$ for these parameters $\pi,\ \tau,\ \tau'$ for simplicity. 
These parameters are normalized so that they satisfy the following condition:
\begin{eqnarray}
  \sum_c \pi_c=1.\ \ \  \sum_i \tau_{c,i}=1.\ \ \  \sum_j \tau'_{c,j}=1.
  \label{eq:normalization}
\end{eqnarray}

Our purpose is to find the parameters $\pi,\ \tau,\ \tau'$ that maximize the likelihood of given matrices $A,\ B$. To solve this problem, we introduce the community assignment $g=\{g_k\}$, where $g_k$ is the community of the $k$-th unit in the depth $d$ layer. The parameters are optimized so that they maximize the likelihood of $A,\ B$ and $g$: 

\begin{eqnarray*}
  \mathrm{Pr}(A,B,g|\pi,\tau,\tau')=\mathrm{Pr}(A,B|g,\pi,\tau,\tau')\ \mathrm{Pr}(g|\pi,\tau,\tau'),
\end{eqnarray*}
where 
\begin{eqnarray*}
  \mathrm{Pr}(A,B|g,\pi,\tau,\tau') &=& \prod_k \left\{ \prod_i {\Bigl(\tau_{g_k,i}\Bigr)}^{A_{i,k}} \right\} \left\{ \prod_j {\Bigl(\tau'_{g_k,j}\Bigr)}^{B_{k,j}} \right\},\\
  \mathrm{Pr}(g|\pi,\tau,\tau') &=& \prod_k \pi_{g_k}.
\end{eqnarray*}
Then, the log likelihood of $A,\ B$ and $g$ is given by
\begin{eqnarray*}
  \mathcal{L} &=& \ln \mathrm{Pr}(A,B,g|\pi,\tau,\tau')\\
  &=& \sum_k \left\{ \ln \pi_{g_k} +\sum_i A_{i,k} \ln \tau_{g_k,i} +\sum_j B_{k,j} \ln \tau'_{g_k,j} \right\}.
\end{eqnarray*}

Here, the community assignment $g$ is a latent variable and is unknown in advance, so we cannot directly calculate the above $\mathcal{L}$. Therefore, we calculate the expected log likelihood $\bar{\mathcal{L}}$ over $g$ instead. 
\begin{eqnarray*}
  \bar{\mathcal{L}} &=& \sum_{g_1}\cdots \sum_{g_l} \mathrm{Pr} (g|A,B,\pi,\tau,\tau') \sum_k \left\{ \ln \pi_{g_k} +\sum_i A_{i,k} \ln \tau_{g_k,i} +\sum_j B_{k,j} \ln \tau'_{g_k,j} \right\}\\
  &=& \sum_{k,c} \mathrm{Pr}(g_k=c|A,B,\pi,\tau,\tau') \left\{ \ln \pi_c +\sum_i A_{i,k} \ln \tau_{c,i} +\sum_j B_{k,j} \ln \tau'_{c,j} \right\},
\end{eqnarray*}
where $l$ is the number of units in the depth $d$ layer. By defining
\begin{eqnarray}
  q_{k,c}=\mathrm{Pr}(g_k=c|A,B,\pi,\tau,\tau')=\frac{\mathrm{Pr}(A,B,g_k=c|\pi,\tau,\tau')}{\mathrm{Pr}(A,B|\pi,\tau,\tau')},
  \label{eq:qdefine}
\end{eqnarray}
the above equation can be rewritten as follows:
\begin{eqnarray}
  \bar{\mathcal{L}}=\sum_{k,c} q_{k,c} \left\{ \ln \pi_c +\sum_i A_{i,k} \ln \tau_{c,i} +\sum_j B_{k,j} \ln \tau'_{c,j} \right\}.
  \label{eq:exploglh}
\end{eqnarray}

The parameter $q_{k,c}$ represents the probability that the $k$-th unit is assigned to the community $c$. In other words, the community detection result is given by the estimated $\{q_{k,c}\}$. The optimal parameters for maximizing $\bar{\mathcal{L}}$ of Eq. (\ref{eq:exploglh}) are found with the EM algorithm. The parameters $\pi,\ \tau,\ \tau'$ with given $\{q_{k,c}\}$ are iteratively optimized.

\begin{theo}
If $\{q_{k,c}\},\ \{\pi_c\},\ \{\tau_{c,i}\},\ \{\tau'_{c,j}\}$ maximizes $\bar{\mathcal{L}}$, then they satisfy 
\begin{eqnarray}
  q_{k,c}=\frac{\pi_c \left[ \prod_i {\tau_{c,i}}^{A_{i,k}} \right] \left[ \prod_j {\tau'_{c,j}}^{B_{k,j}} \right]}{\sum_s \pi_s \left[ \prod_i {\tau_{s,i}}^{A_{i,k}} \right] \left[ \prod_j {\tau'_{s,j}}^{B_{k,j}} \right]},\ (\forall k,\ c)
  \label{eq:q}
\end{eqnarray}
and
\begin{eqnarray}
  \pi_c &=& \frac{\sum_k q_{k,c}}{l},\nonumber \\
  \tau_{c,i} &=& \frac{\sum_k q_{k,c} A_{i,k}}{\sum_{k,i} q_{k,c} A_{i,k}},\nonumber \\
  \tau'_{c,j} &=& \frac{\sum_k q_{k,c} B_{k,j}}{\sum_{k,j} q_{k,c} B_{k,j}}.\ (\forall c,\ i,\ j)
  \label{eq:pitau}
\end{eqnarray}
\end{theo}

\begin{proof}
The denominator and numerator in the last term of Eq. (\ref{eq:qdefine}) are given by
\begin{eqnarray*}
  \lefteqn{ \mathrm{Pr}(A,B,g_k=c|\pi,\tau,\tau') = \sum_{g_1}\cdots \sum_{g_l} \delta_{g_k,c} \mathrm{Pr}(A,B,g|\pi,\tau,\tau') }\\
  &=& \sum_{g_1}\cdots \sum_{g_l} \delta_{g_k,c} \prod_h \left\{ \pi_{g_h} \left[ \prod_i {\tau_{g_h,i}}^{A_{i,h}} \right] \left[ \prod_j {\tau'_{g_h,j}}^{B_{h,j}} \right] \right\}\\
  &=& \left\{ \pi_c \left[ \prod_i {\tau_{c,i}}^{A_{i,k}} \right] \left[ \prod_j {\tau'_{c,j}}^{B_{k,j}} \right] \right\} \left\{ \prod_{h\neq k} \sum_s \pi_s \left[ \prod_i {\tau_{s,i}}^{A_{i,h}} \right] \left[ \prod_j {\tau'_{s,j}}^{B_{h,j}} \right] \right\},
\end{eqnarray*}
and
\begin{eqnarray*}
  \mathrm{Pr}(A,B|\pi,\tau,\tau') &=& \sum_{g_1}\cdots \sum_{g_l} \mathrm{Pr}(A,B,g|\pi,\tau,\tau')\\
  &=& \prod_k \sum_s \pi_s \left[ \prod_i {\tau_{s,i}}^{A_{i,k}} \right] \left[ \prod_j {\tau'_{s,j}}^{B_{k,j}} \right],
\end{eqnarray*}
where $\delta_{i,j}$ is the Kronecker delta. Therefore, $q_{k,c}$ is given by Eq. (\ref{eq:q}).

The problem is to maximize $\bar{\mathcal{L}}$ of Eq. (\ref{eq:exploglh}) with a given $\{q_{k,c}\}$ under the condition of Eq. (\ref{eq:normalization}). This is solved with the Lagrangian undetermined multiplier method, which employs
\begin{eqnarray*}
  f = \bar{\mathcal{L}}-\alpha \sum_c \pi_c-\sum_c \beta_c \sum_i \tau_{c,i}-\sum_c \gamma_c \sum_j \tau'_{c,j},
\end{eqnarray*}
and
\begin{eqnarray}
  \frac{\partial f}{\partial \pi_c}=\frac{\partial f}{\partial \tau_{c,i}}=\frac{\partial f}{\partial \tau'_{c,j}}=0.\ (\forall c,\ i,\ j)
  \label{eq:lag1}
\end{eqnarray}
From Eq. (\ref{eq:lag1}), the following equations are derived:
\begin{eqnarray}
\frac{\partial \bar{\mathcal{L}}}{\partial \pi_c}=\alpha,\ \ \ 
\frac{\partial \bar{\mathcal{L}}}{\partial \tau_{c,i}}=\beta_c,\ \ \ 
\frac{\partial \bar{\mathcal{L}}}{\partial \tau'_{c,j}}=\gamma_c.\ (\forall c,\ i,\ j)
  \label{eq:lag2}
\end{eqnarray}
Using Eq. (\ref{eq:exploglh}) and Eq. (\ref{eq:lag2}), we obtain 
\begin{eqnarray}
\pi_c = \frac{1}{\alpha}\sum_k q_{k,c},\ \ 
\tau_{c,i} = \frac{1}{\beta_c} \sum_k q_{k,c} A_{i,k},\ \ 
\tau'_{c,j} = \frac{1}{\gamma_c} \sum_k q_{k,c} B_{k,j}.\ (\forall c,\ i,\ j)
  \label{eq:lag3}
\end{eqnarray}
From Eq. (\ref{eq:lag3}) and the condition of Eq. (\ref{eq:normalization}), Lagrange's undetermined multipliers $\alpha,\ \{\beta_c\},\ \{\gamma_c\}$ are determined, and Eq. (\ref{eq:lag3}) is rewritten as Eq. (\ref{eq:pitau}). 
\end{proof}

From the above theorem, the optimal parameters $\pi, \tau, \tau'$ and the probability of community assignment $q$ for the optimized parameters are iteratively estimated based on Eq. (\ref{eq:q}) and (\ref{eq:pitau}). In this paper, the community assigned to the $k$-th unit is determined by the $c$ that maximizes $q_{k,c}$ (Figure \ref{fig:com} (C)).

Finally, we use the following methods to determine a modular representation of a layered neural network that summarizes multiple connections between the pairs of communities (Figure \ref{fig:com} (D)). \\
\\
\textbf{Four Algorithms for Determining Bundled Connections}:
\begin{itemize}
\item Method $1$: Community $a$ and $b$ have a bundled connection iff there exists at least one connection between the pairs of units $\{i,j\},\ i\in a,\ j\in b$.
\item Method $2$: Let the number of units in communities $a$ and $b$ be $l_a$ and $l_b$, respectively, and let the number of connections between the pairs of units $\{i,j\},\ i\in a,\ j\in b$ be $l_{a,b}$. 
Communities $a$ and $b$ have a bundled connection iff $r_{a,b}\equiv \frac{l_{a,b}}{l_a l_b}\geq \zeta$ holds, where $\zeta$ is a threshold.
\item Method $3$: Among the bundled connections defined by Method $2$, only those that satisfy the following (1) OR (2) are kept and the others are removed. (1) for any community $a'$ in the same layer as community $a$, $r_{a,b}\geq r_{a',b}$. (2) for any community $b'$ in the same layer as community $b$, $r_{a,b}\geq r_{a,b'}$.
\item Method $4$: Among the bundled connections defined by Method $2$, only those that satisfy the above (1) AND (2) are kept and the others are removed. 
\end{itemize}

By these procedures, we obtain the modular representation of a layered neural network. 

\section{Experiments} 
\label{sec:experiment}

In this section, we show three applications of the proposed method: (1) the decomposition of a layered neural network into independent networks, (2) generalization error estimation from a community structure, and (3) knowledge discovery from a modular representation. Here we verify the effectiveness of the proposed method in the above three applications. 

The following processing was performed in all the experiments:
\begin{enumerate}
\setlength{\itemsep}{0cm}
\setlength{\parskip}{0cm}
\renewcommand{\labelenumi}{(\arabic{enumi})}
\item The input data were normalized so that the minimum and maximum values were $x_{\mathrm{min}}$ and $x_{\mathrm{max}}$, respectively.
\item The output data were normalized so that the minimum and maximum values were $0.01$ and $0.99$, respectively.
\item The initial parameters were independently generated as follows: $\omega^d_{ij}\overset{\text\small\rm{i.i.d.}}{\sim} \mathcal{N}(0,0.5)$. $\theta^d_j\overset{\text\small\rm{i.i.d.}}{\sim} \mathcal{N}(0,0.5)$.
\item As in Eq. (\ref{eq:Aelement}),  the connection matrix $A^d_{ij}=0.99$ if the absolute value of the connection weight between the $i$-th unit in the depth $d-1$ layer and the $j$-th unit in the depth $d$ layer is larger than a threshold $\xi$, otherwise $A^d_{ij}=0.01$. Note that $0.99$ and $0.01$ are used instead of $1$ and $0$ for stable computation. Similarly, $B^d_{ij}$ is defined from the connection weight between the $i$-th unit in the depth $d$ layer and the $j$-th unit in the depth $d+1$ layer (Eq. (\ref{eq:Belement})). All units were removed that had no connections to other units.
\item For each layer in a trained neural network, $10$ community detection trials were performed. We defined the community detection result as one that achieved the largest expected log likelihood in the last of $200$ iterations of the EM algorithm.
\item In each community detection trial, the initial values of the parameters $\pi,\ \tau,\ \tau'$ were independently generated from a uniform distribution on $(0,1)$, and then normalized so that Eq. (\ref{eq:normalization}) held. 
\item In visualization of modular representation, all communities with no output bundled connections from them were regarded as unnecessary communities. 
In the output layer, the communities with no input bundled connections were regarded in the same way as above. The bundled connections with such unnecessary communities were also removed. 
These unnecessary communities and bundled connections were detected from depth $D$ to $1$, since the unnecessary communities in the shallower layers depend on the removal of unnecessary bundled connections in the deeper layers. 
\end{enumerate}

\subsection{Decomposition of independent layered neural networks} 
\label{sec:decomposition}

We show that the proposed method can properly decompose a neural network into a set of small independent neural networks, where the data set consists of multiple independent dimensions. For validation, we made synthetic data of three independent parts, merged them, and applied the proposed method to decompose them into the three independent parts. 

\subsubsection{Generation method of synthetic data}
\label{sec:independent}

The method we used to generate the synthetic data is shown in Figure \ref{fig:exp1}. In the following, we explain the experimental settings in detail.

\begin{figure*}
  \centering
  \includegraphics[width=125mm]{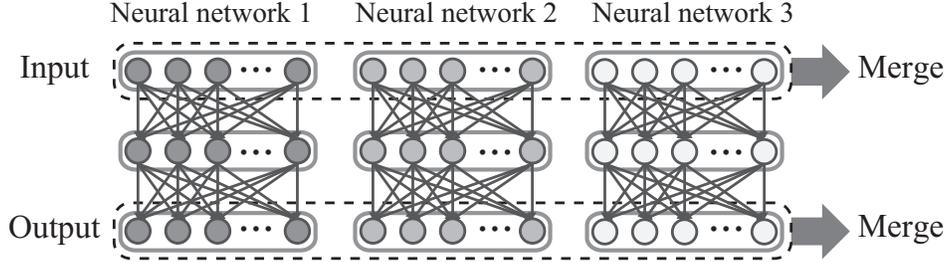}
  \caption{Method for generating synthetic data. First, three vectors of input data were independently generated. Then, each input vector was connected to a different layered neural network with independent weights and biases. Independent noises were added to the resulting three output vectors, to generate the output data. These three sets of input and output data were merged into one.}
  \label{fig:exp1}
\end{figure*}

First, three sets of input data were independently generated. All the sets contained input data with $15$ dimensions, and their values followed: $x^n_j\overset{\text\small\rm{i.i.d.}}{\sim}\mathcal{N}(0,3)$. 
Then, three neural networks were defined, each of which has independent weights and biases. In each neural network, all the layers consisted of $15$ units, and the number of hidden layers was set at one. The sets of weights and biases for the first, second and third neural networks are denoted as $\{\omega,\ \theta\}$, $\{\omega',\ \theta'\}$, and $\{\omega'',\ \theta''\}$, respectively. These parameters were randomly generated as follows: 
\begin{eqnarray*}
  \omega^d_{i,j},\ {\omega'}^d_{i,j},\ {\omega''}^d_{ij}\overset{\text\small\rm{i.i.d.}}{\sim}\mathcal{N}(0,2),\\
  \theta^d_j,\ {\theta'}^d_j,\ {\theta''}^d_j\overset{\text\small\rm{i.i.d.}}{\sim}\mathcal{N}(0,0.5).
\end{eqnarray*}
For the weights $\omega,\ \omega'$ and $\omega''$, the connections with absolute values of one or smaller were replaced by $0$. 

Finally, three sets of output data were generated by using the above input data and neural networks by adding independent noise following $\mathcal{N}(0,0.05)$. The three generated sets of input and output data were merged into one set of data, as shown in Figure \ref{fig:exp1}. 

\subsubsection{Neural network training and modular representation extraction}
\label{sec:independentmr}

We trained another neural network with $45$ dimensions for input, hidden and output layer using the merged data. 
Then, a modular representation of the trained neural network was made with the proposed method. 
The results of the trained neural network, its community structure, and its modular representation are shown in Figures \ref{fig:exp1o}, \ref{fig:exp1c}, and \ref{fig:exp1m}, respectively. The numbers above the input layer and below the output layer are the indices of the three sets of data. These results showed that the proposed method could decompose the trained neural network into three independent networks. 

\begin{figure*}
  \centering
  \includegraphics[width=100mm]{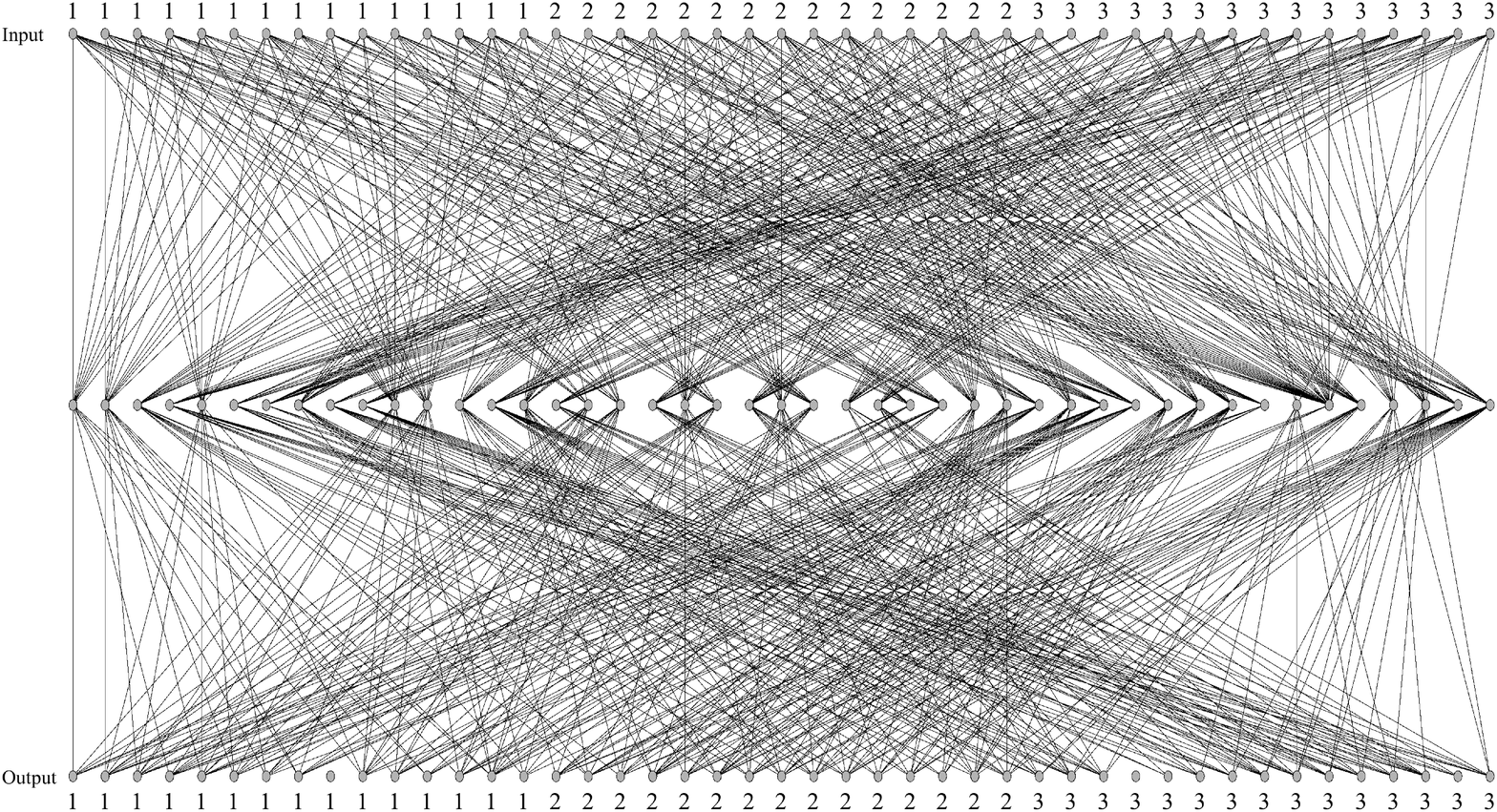}\vspace{-6mm}
  \caption{Neural network trained using the synthetic data. The numbers above the input layer and below the output layer are the indices of the three data sets.}\vspace{1mm}
  \label{fig:exp1o}
  \centering
  \includegraphics[width=100mm]{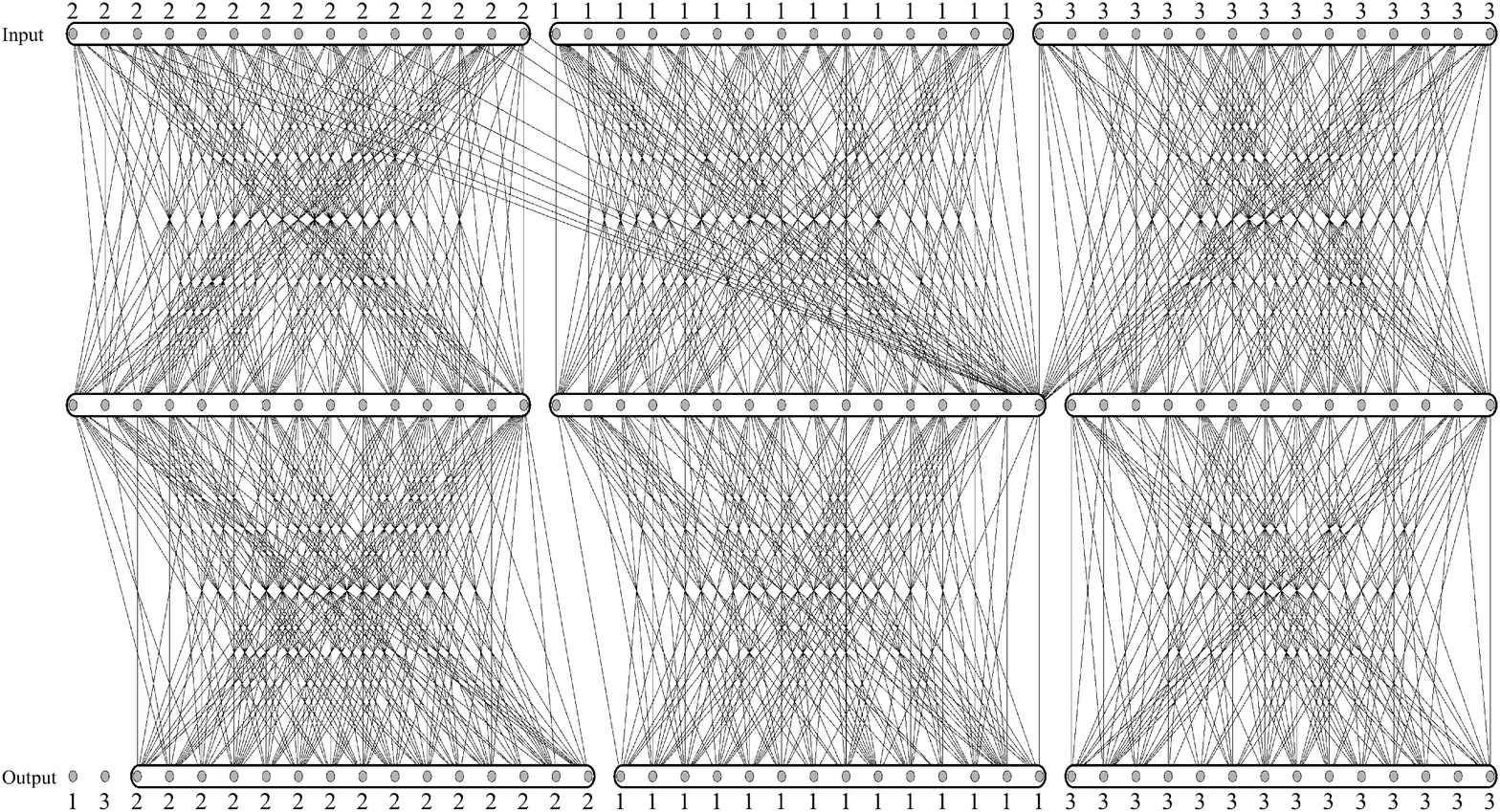}\vspace{-6mm}
  \caption{Result of community detection with proposed method.}\vspace{-1mm}
  \label{fig:exp1c}
  \centering
  \includegraphics[width=100mm]{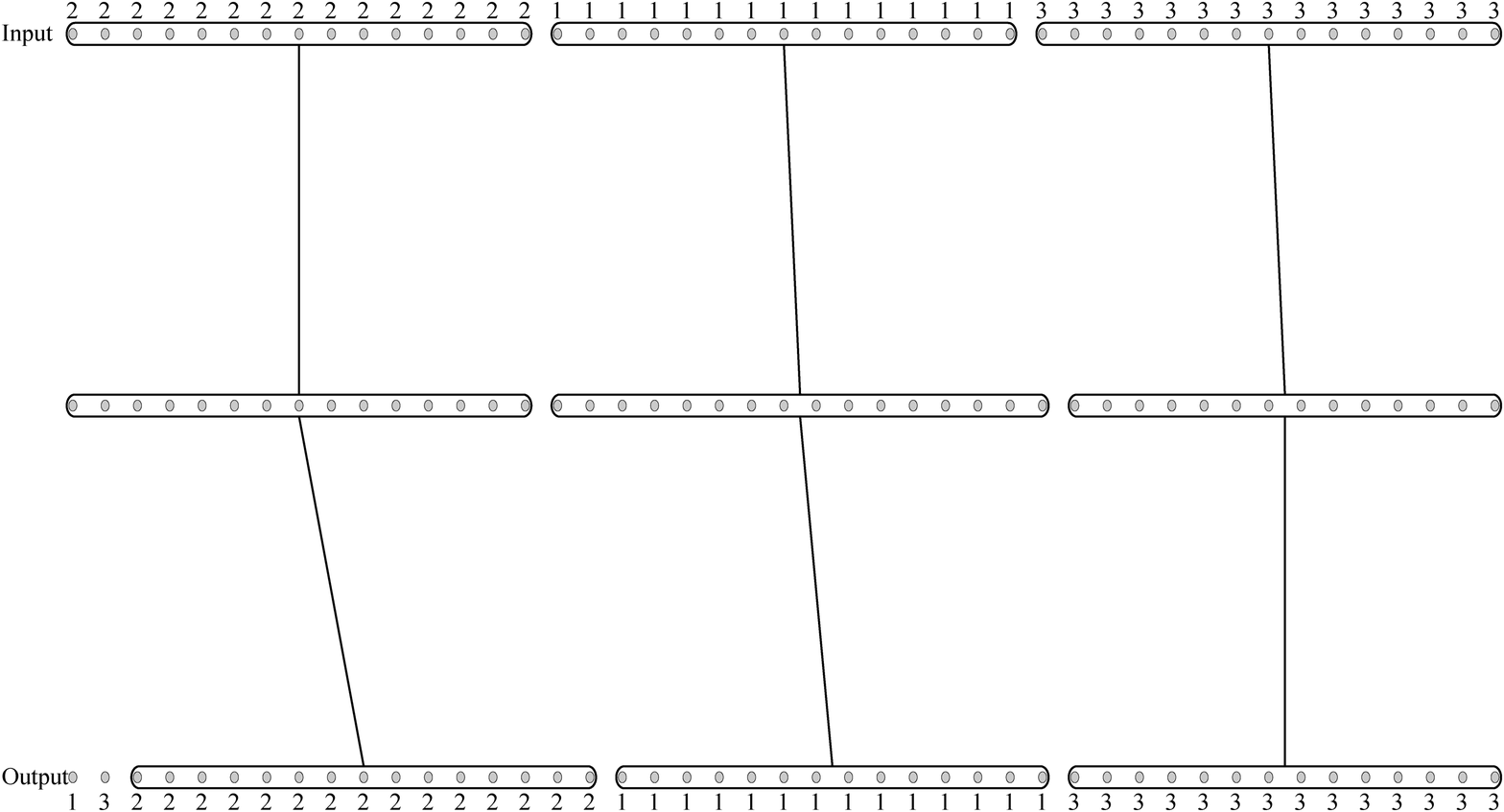}\vspace{-6mm}
  \caption{Extracted modular representation of trained neural network. These results showed that our method could decompose the trained neural network into three independent networks.}
  \label{fig:exp1m}
\end{figure*}

\subsubsection{Modular representation extraction using data generated by mutually dependent neural networks}

We also conducted an experiment to extract modular representations from neural networks trained with data generated by mutually dependent neural networks. To control the extent of independence between three neural networks, we added $\kappa$ bundled connections between randomly chosen pairs of communities in mutually adjacent layers, where $\kappa$ is set at $1,\ 2,\ \cdots,\ 10$. For example, in Figure \ref{fig:exp1}, the community in the input layer of neural network 1 and the community in the hidden layer of neural network 2 are randomly chosen and and a bundled connection is added between them. The connection weights between a pair of communities with a bundled connection are independently generated from $\mathcal{N}(0,2)$. As in the experiment described in section \ref{sec:independent}, output data of three communities were generated by using the above neural network with additional bundled connections and three sets of input data that were independently generated. 

As in section \ref{sec:independentmr}, a modular representations of the trained neural networks were extracted with the proposed method, with varying number of true additional bundled connections ($\kappa =1,\ 2,\ \cdots,\ 10$). The results of the modular representations are shown in Figure \ref{fig:exp1bc}. These results showed that our proposed method could almost properly decompose the units in all layers when $\kappa \leq 7$ holds. As the number of additional bundled connections increases ($\kappa \geq 8$), a pair of communities in the same layer is more likely to share much connections to other units in the ground truth neural network, resulting that the units in such ground truth communities cannot be decomposed properly.

\begin{figure*}
  \centering
  \includegraphics[width=62mm]{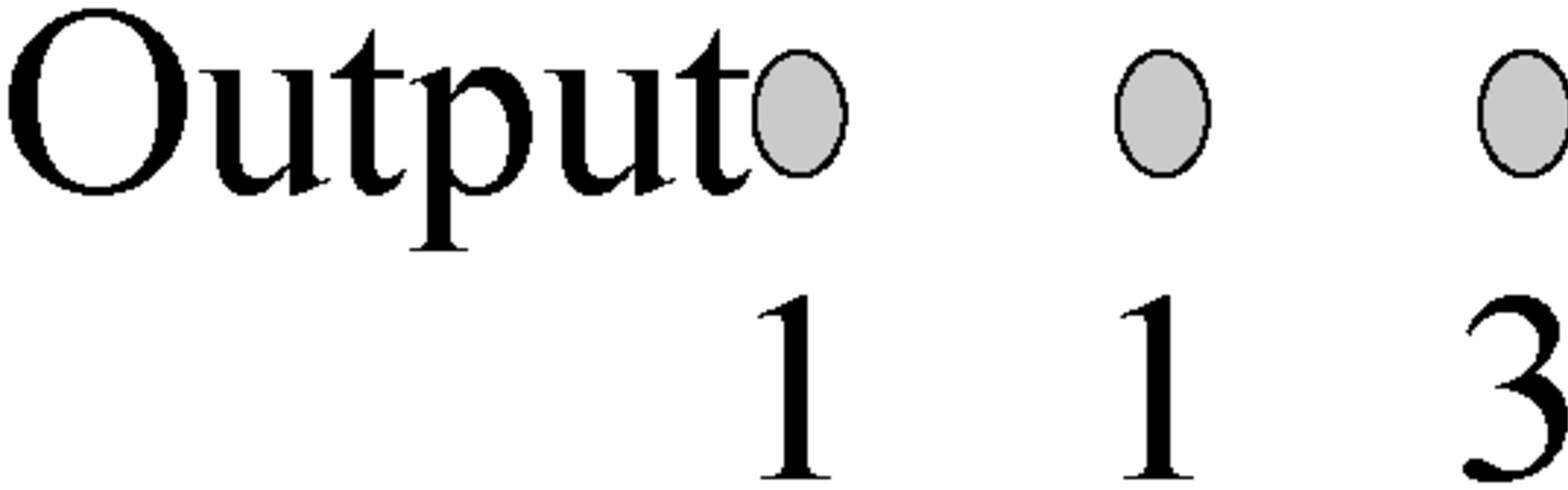}\hspace{3mm}
  \includegraphics[width=62mm]{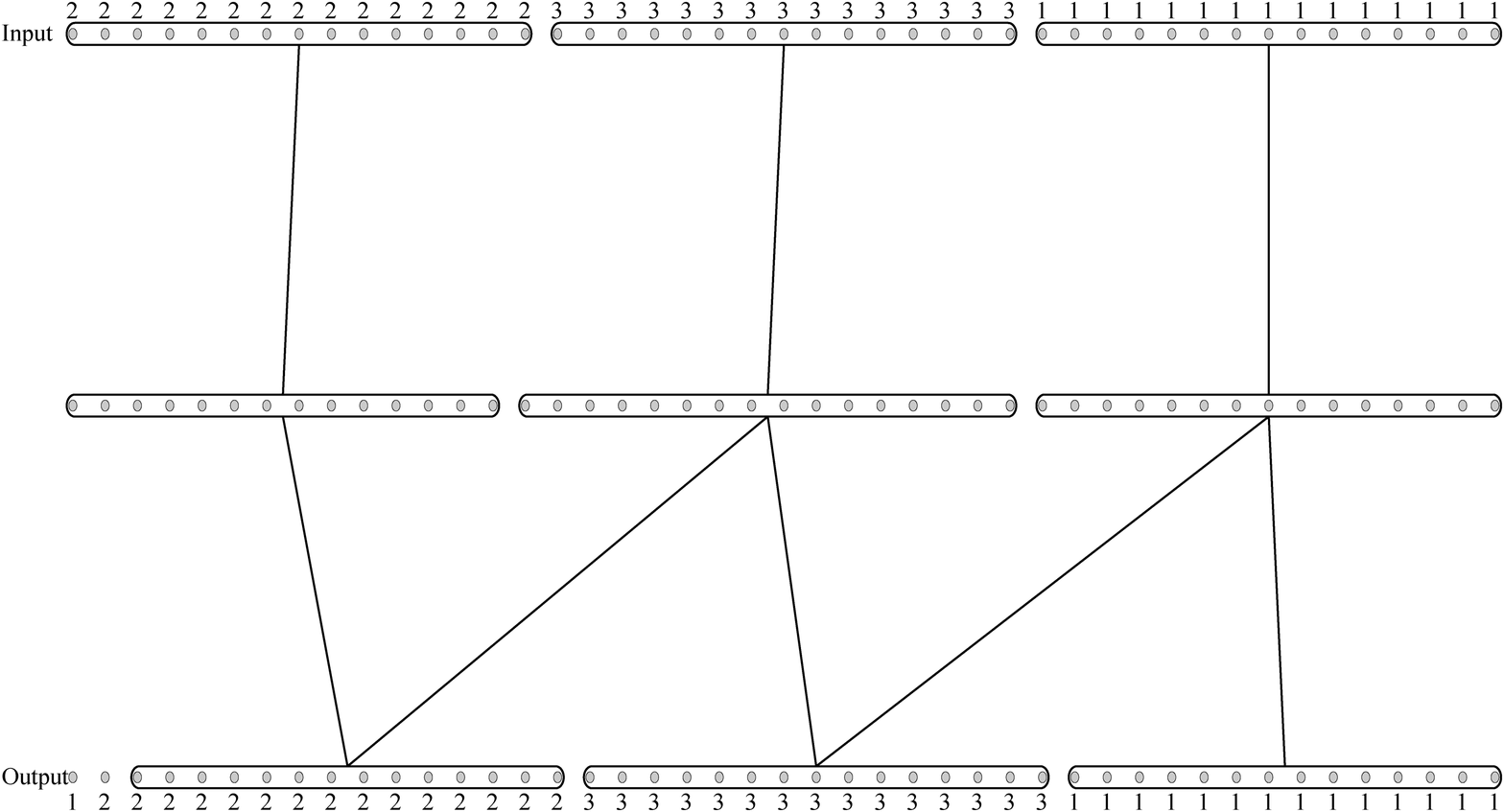}\vspace{1mm} \\
  \includegraphics[width=62mm]{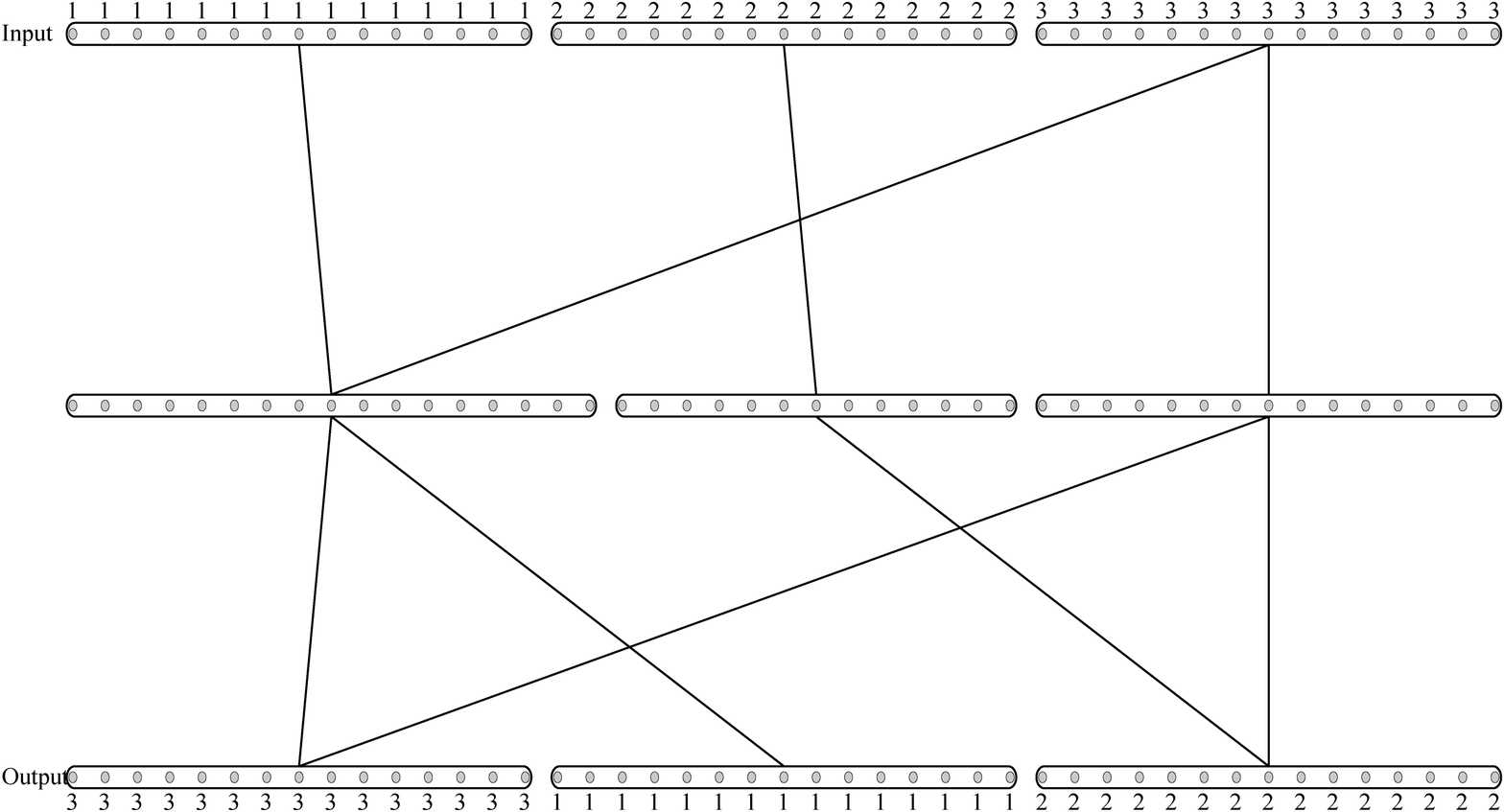}\hspace{3mm}
  \includegraphics[width=62mm]{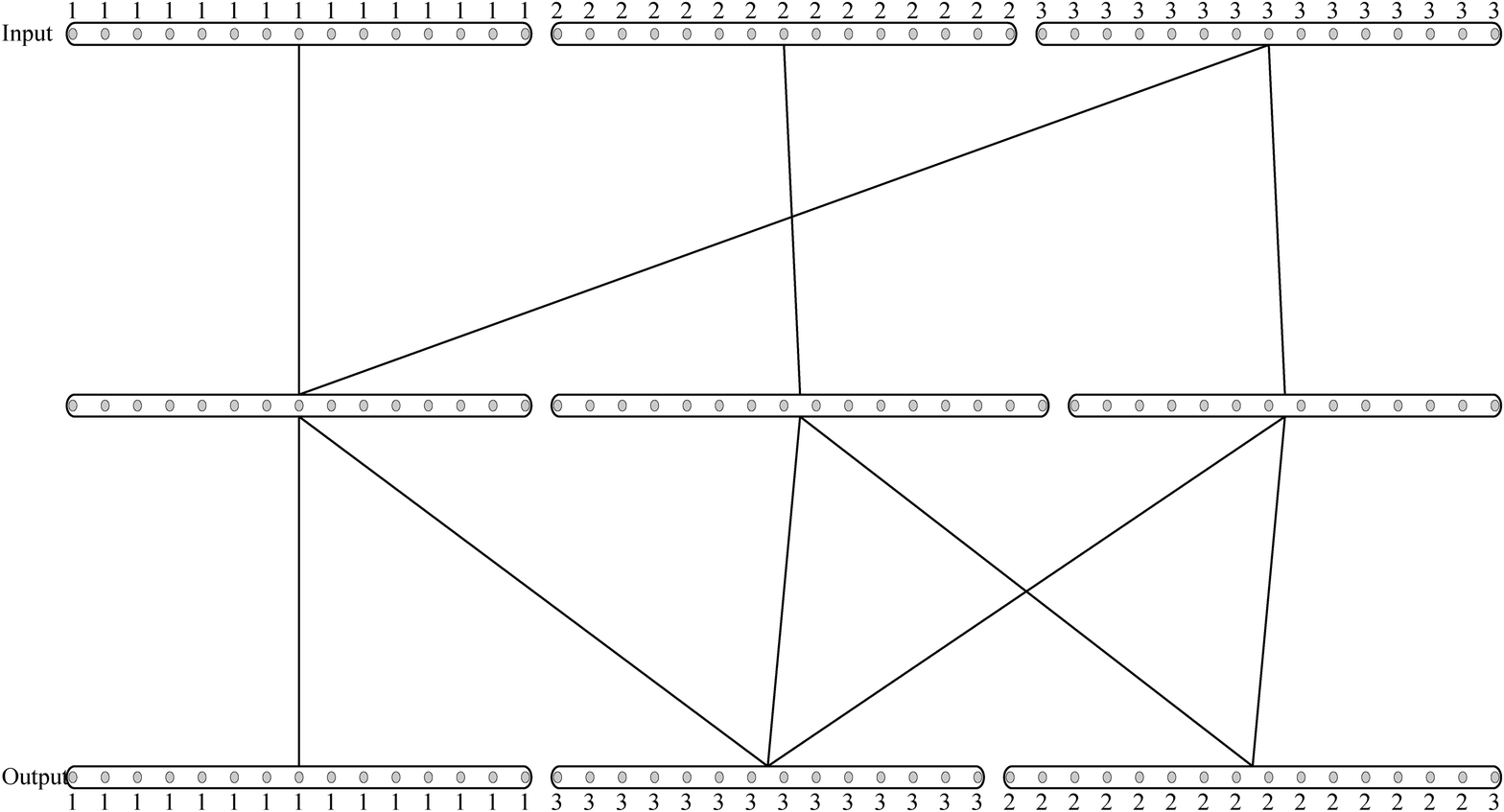}\vspace{1mm} \\
  \includegraphics[width=62mm]{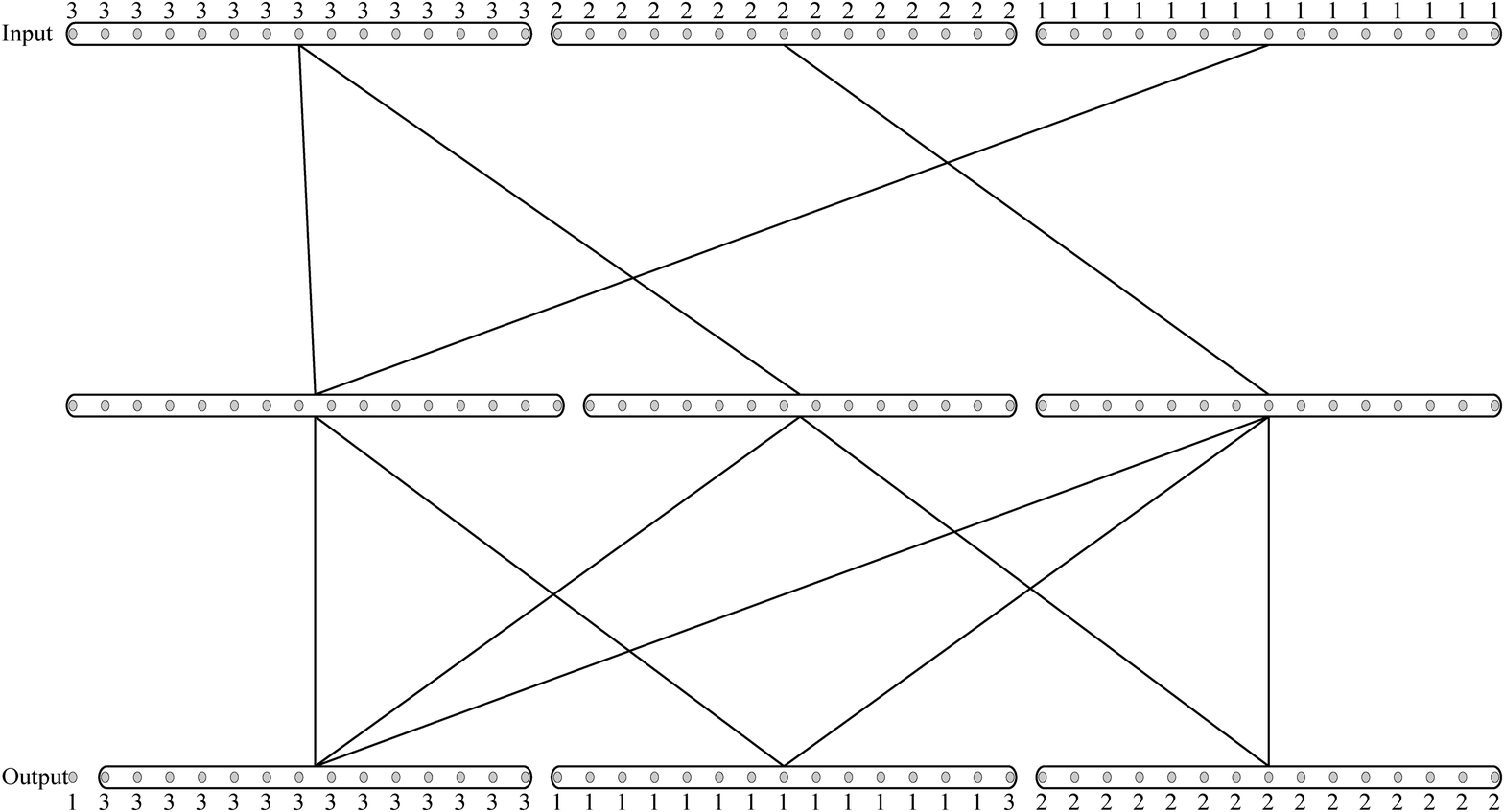}\hspace{3mm}
  \includegraphics[width=62mm]{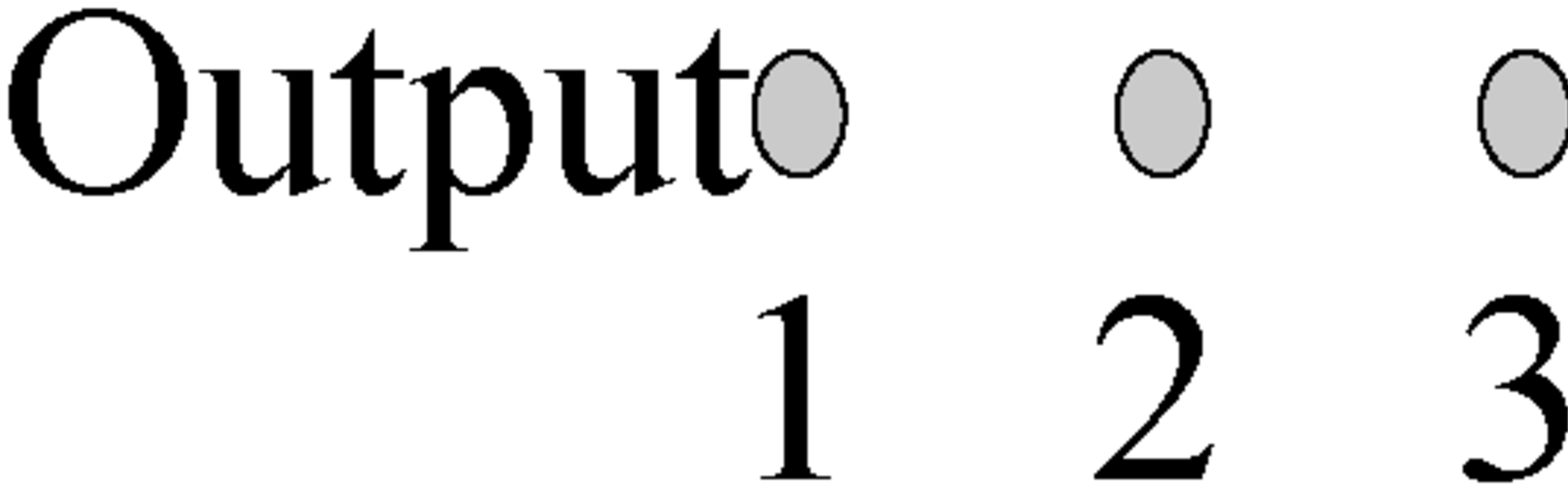}\vspace{1mm} \\
  \includegraphics[width=62mm]{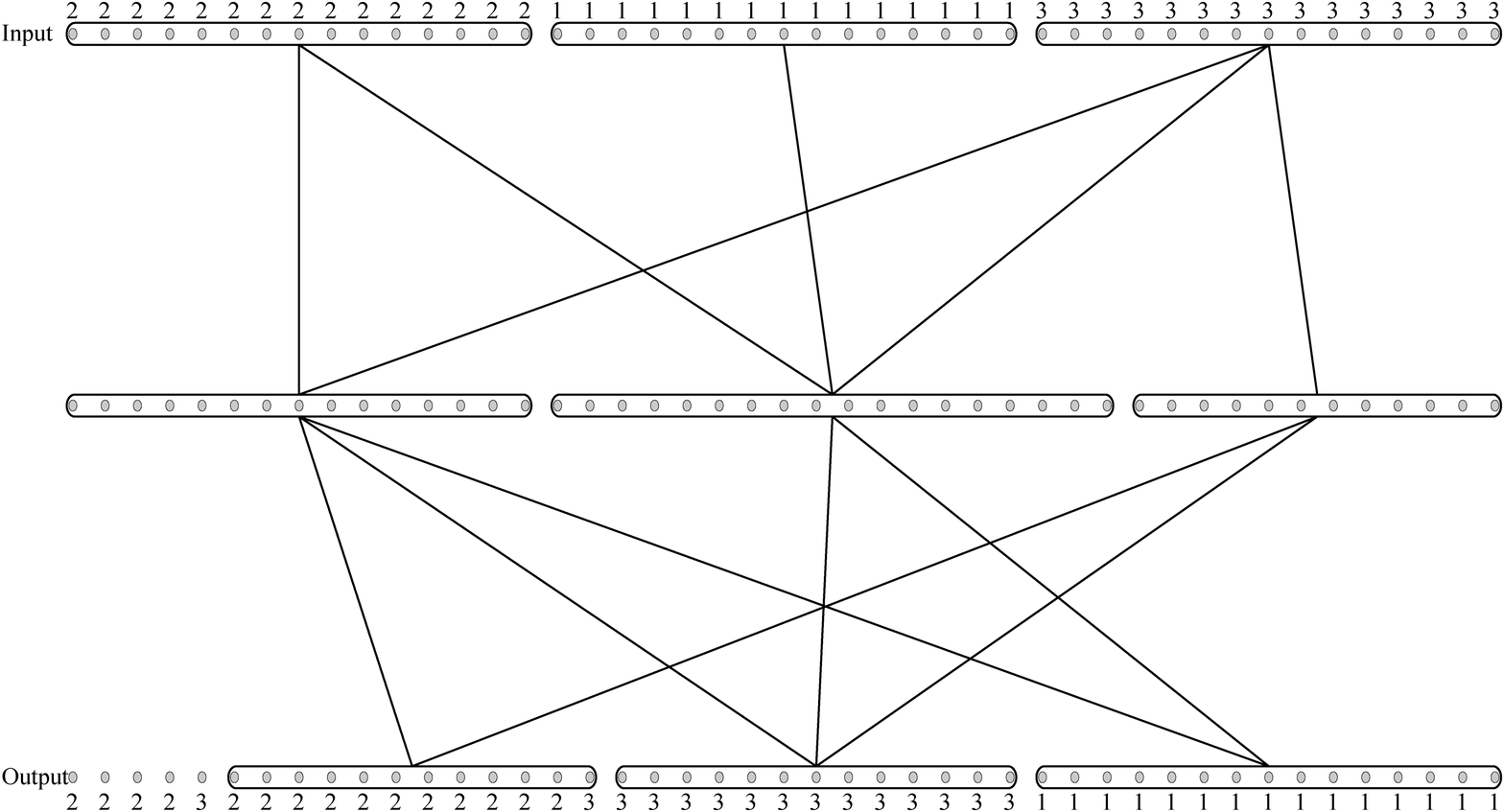}\hspace{3mm}
  \includegraphics[width=62mm]{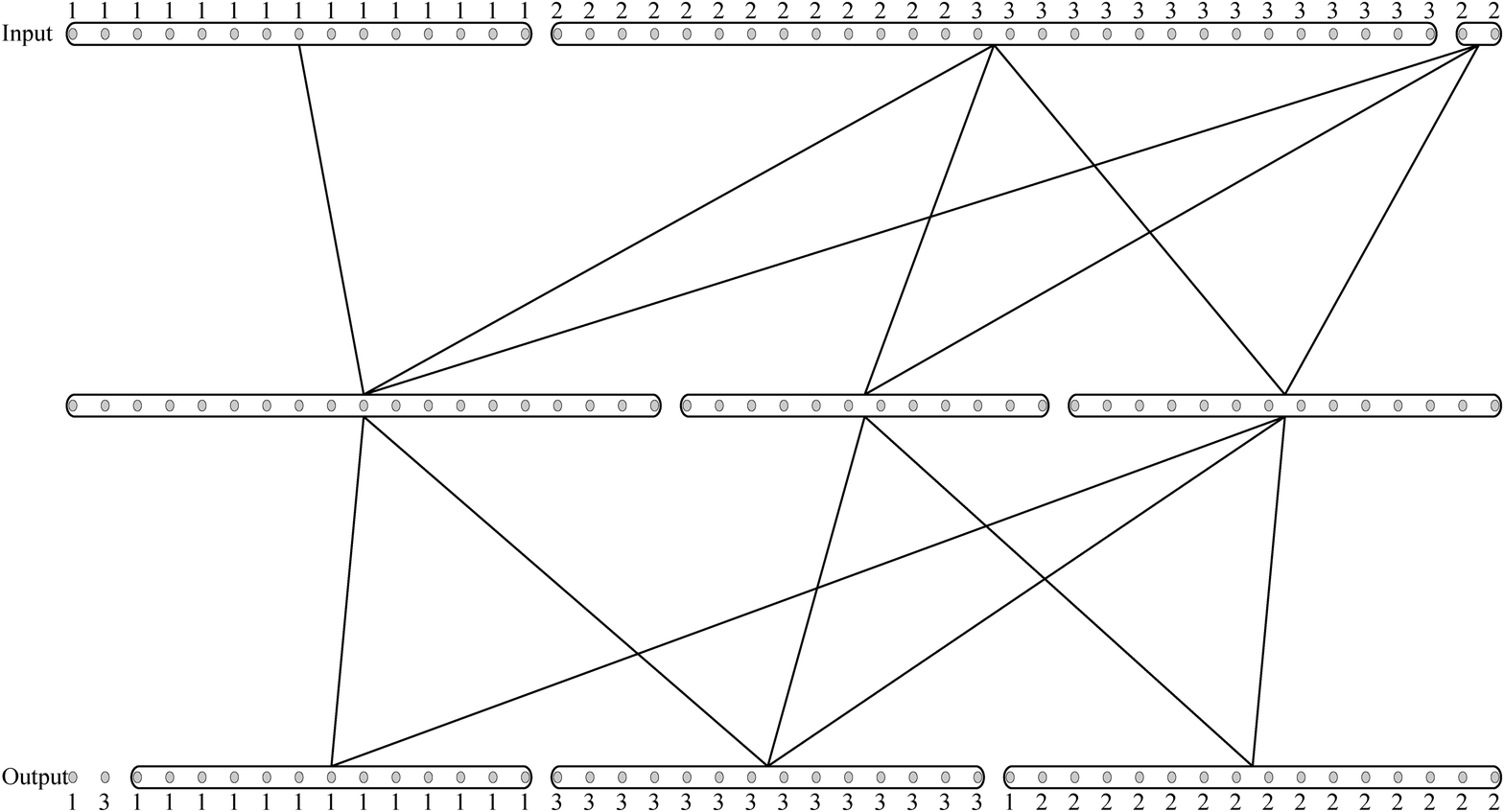}\vspace{1mm} \\
  \includegraphics[width=62mm]{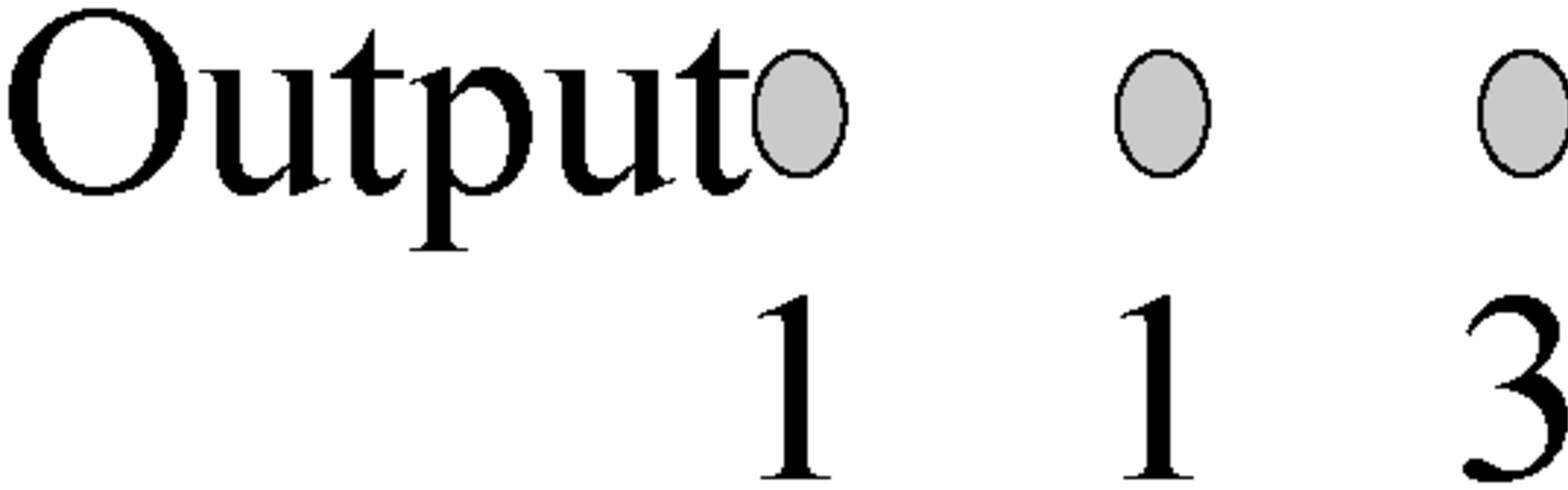}\hspace{3mm}
  \includegraphics[width=62mm]{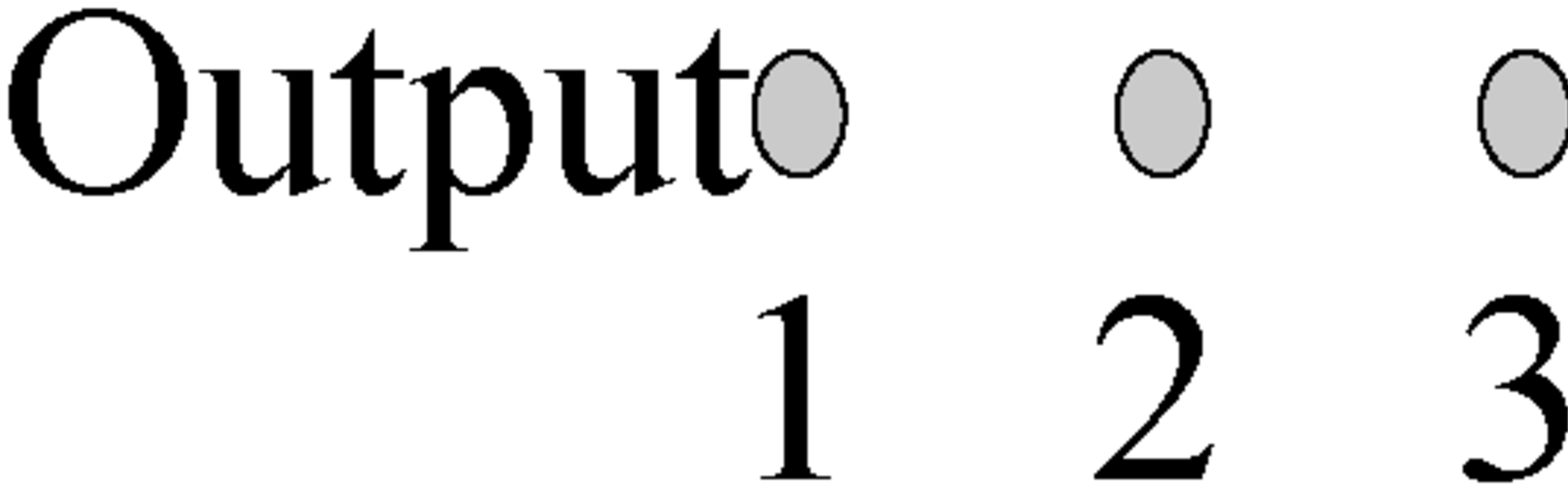}\vspace{-3mm}
  \caption{Extracted modular representation of neural network trained using the data with varying true structure. From the top left, the number of true additional bundled connections is $1,\ 2,\ \cdots,\ 10$.}
  \label{fig:exp1bc}
\end{figure*}

\subsection{Generalization error estimation from community structure} 
\label{sec:gee}

In general, a trained result of a layered neural network is affected by different hyperparameters and initial parameter values. Here, we show that the appropriateness of a trained result can be estimated from the extracted community structure by checking the correlation between the generalization error and the modularity \cite{Newman2004}.

The modularity is defined as a measure of the effectiveness of the community detection result, and it becomes higher with more intra-community connections and fewer inter-community connections. In other words, a network can be divided into different communities more clearly, as the modularity becomes higher. Let the number of communities in the network be $C$, and $\bar{A}=\{\bar{A}_{ij}\}$ be a $C\times C$ matrix whose element $\bar{A}_{ij}$ is the number of connections between communities $i$ and $j$, divided by the total number of connections in the network. The modularity $Q$ of the network is defined by
\begin{eqnarray*}
  Q=\sum_i \Bigl(\bar{A}_{ii}-\left\{\sum_j \bar{A}_{ij}\right\}^2\Bigr).
\end{eqnarray*}
This is a measure for verifying the community structure of assortative networks, so it cannot be applied directly to layered neural networks. In this paper, we define a modified adjacency matrix based on the original adjacency matrix of a layered neural network, and use it for measuring modularity. In the modified adjacency matrix, an element indexed by row $i$ and column $j$ represents the number of common units that connect with both the $i$-th and $j$-th units (Figure \ref{fig:modularitycalc}). We set the diagonal elements of modified adjacency matrix at $0$, resulting that there are no self-loops.

\begin{figure*}
  \centering
  \includegraphics[width=135mm]{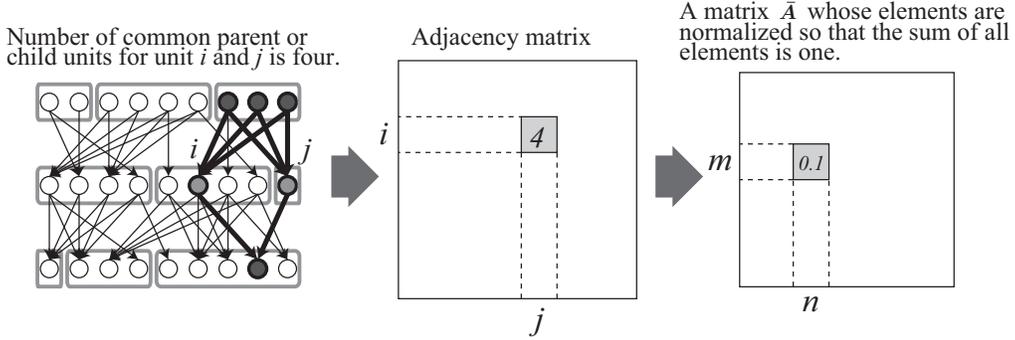}
  \caption{Left and center: method for defining the modified adjacency matrix of a layered neural network for calculating modularity. Right: A matrix $\bar{A}$ whose elements indicate the fraction of the connection weights between two communities in the network.}
  \label{fig:modularitycalc}
\end{figure*}

\subsubsection{Correlation between modularity and generalization error when using input data of mutually independent dimensions}

If the data consist of multiple independent dimensions like the synthetic data used in the experiment described in section \ref{sec:decomposition}, the generalization error is expected to be smaller when the weights of the connections between independent sets of input and output are trained to be smaller. Therefore, a higher modularity indicates a smaller generalization error.

\begin{figure*}
\begin{minipage}[t]{135mm}
  \centering
  \includegraphics[width=35mm]{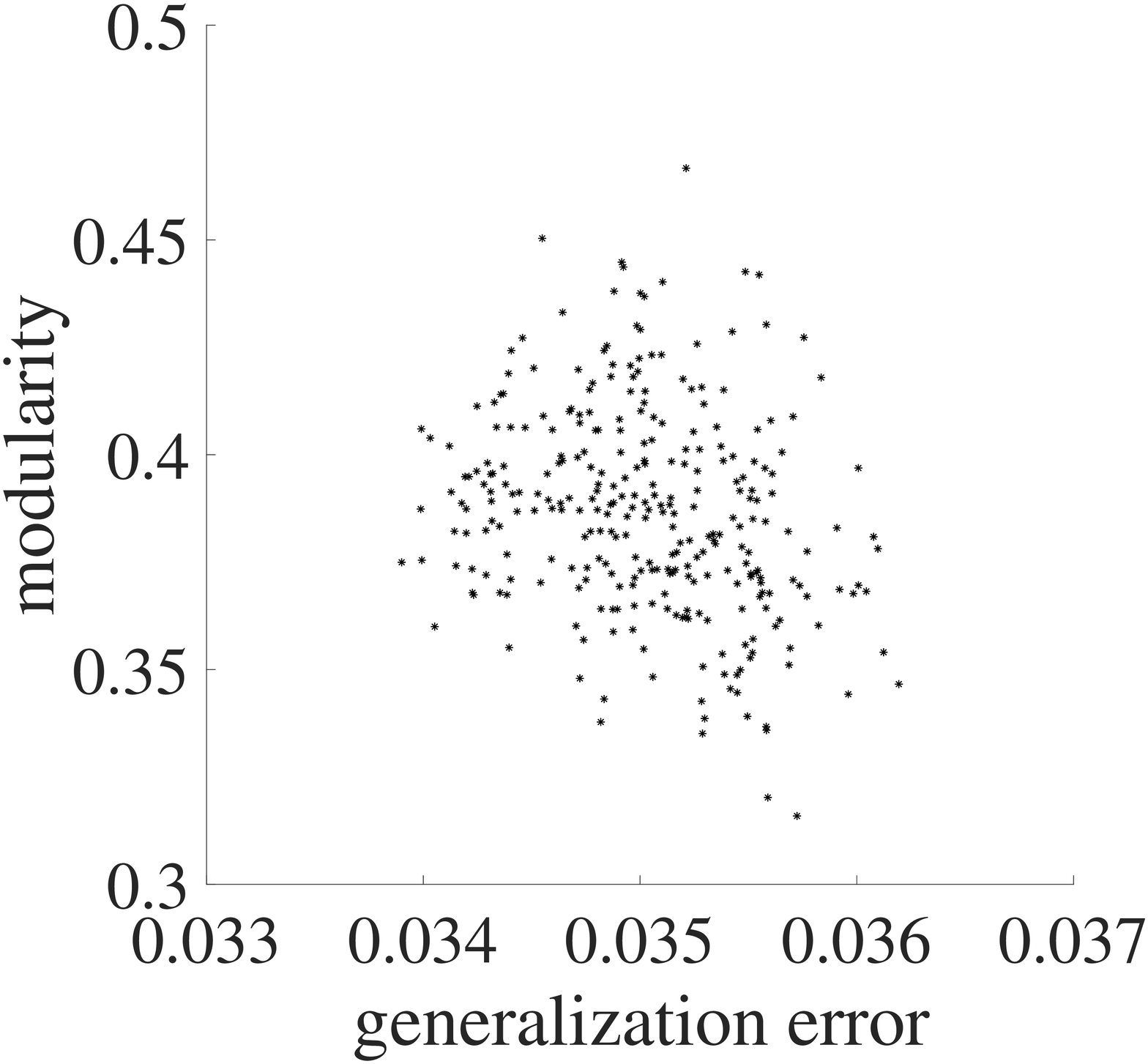}
  \includegraphics[width=35mm]{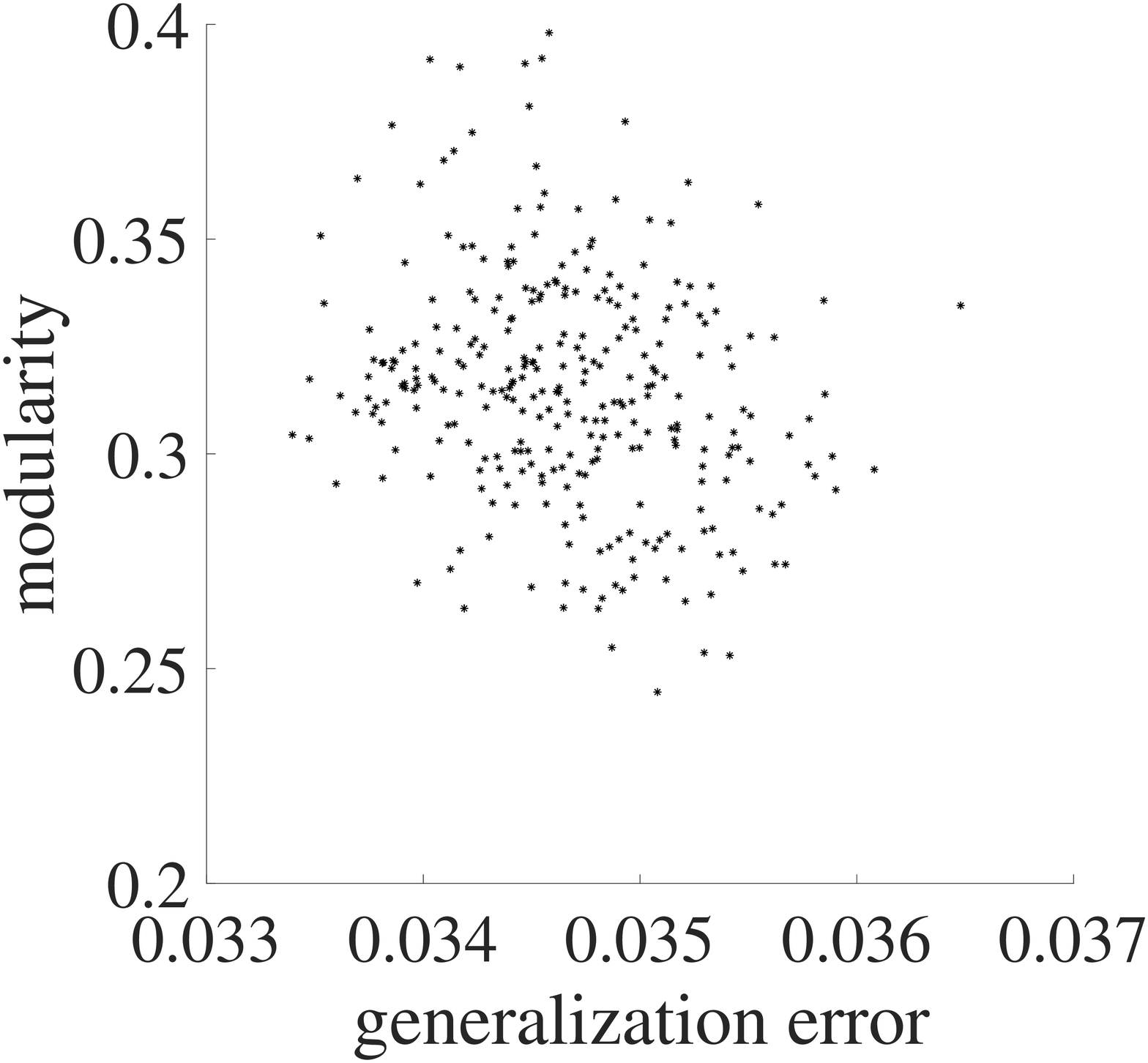}
  \includegraphics[width=35mm]{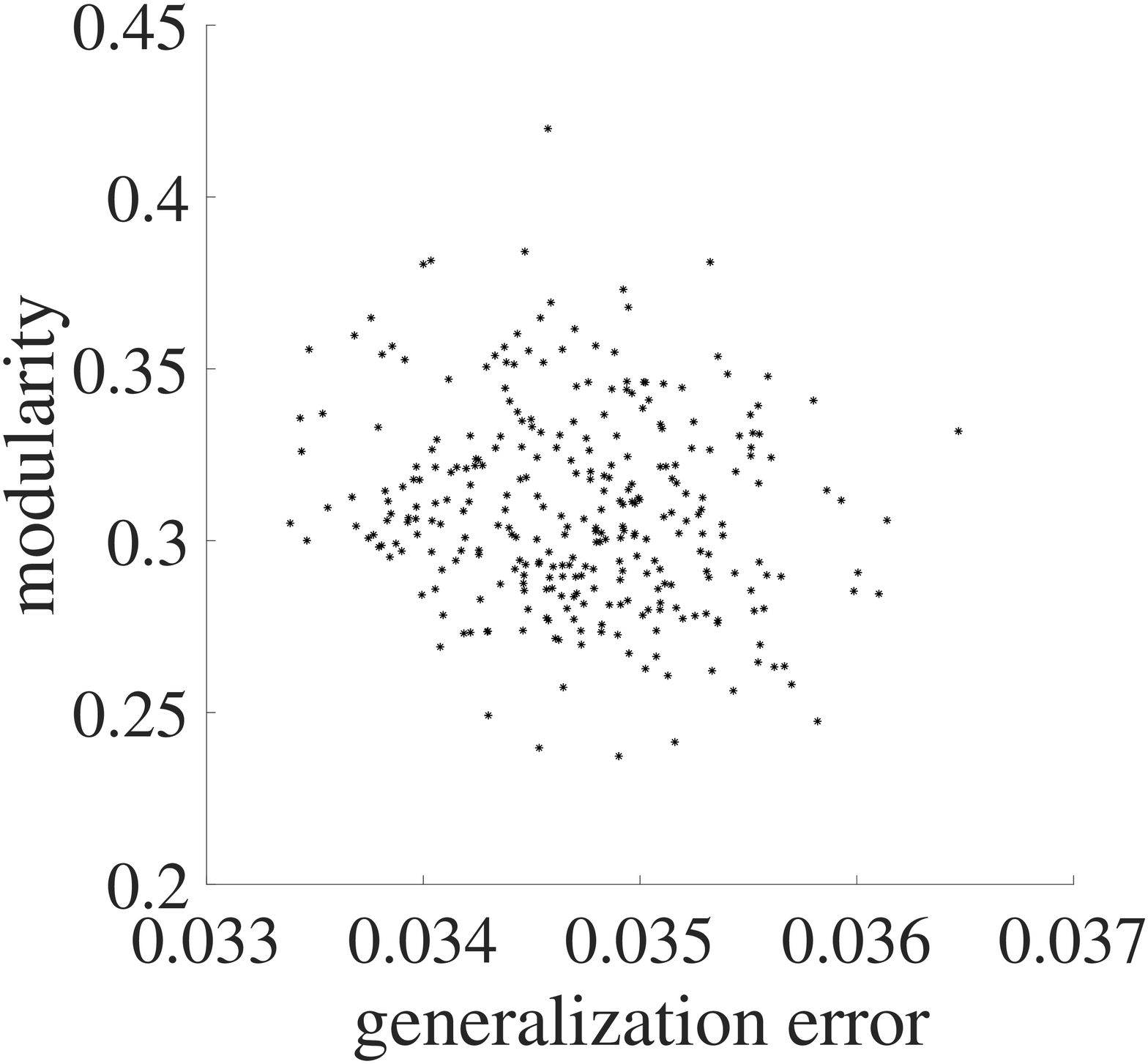}\\
  \includegraphics[width=35mm]{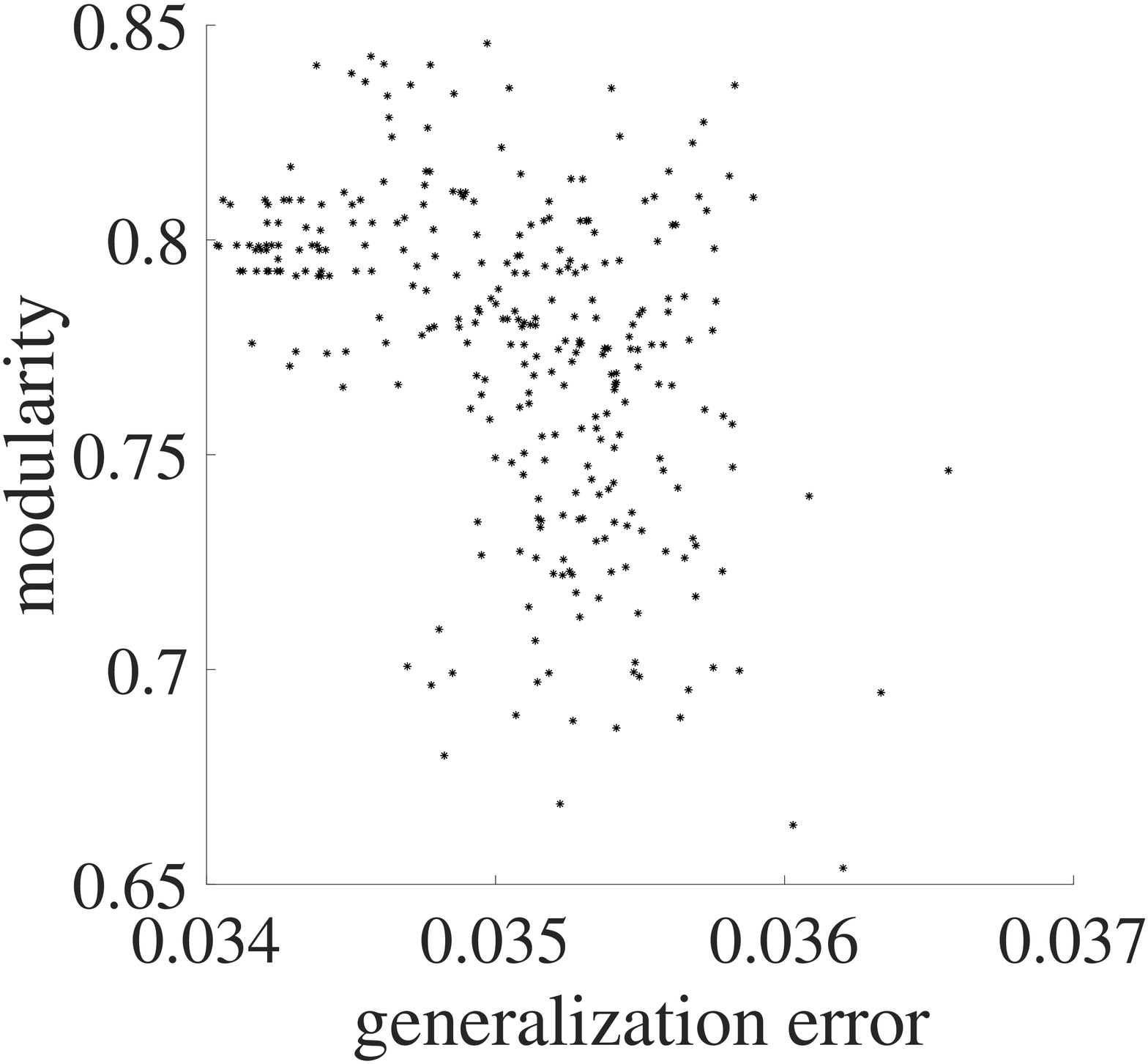}
  \includegraphics[width=35mm]{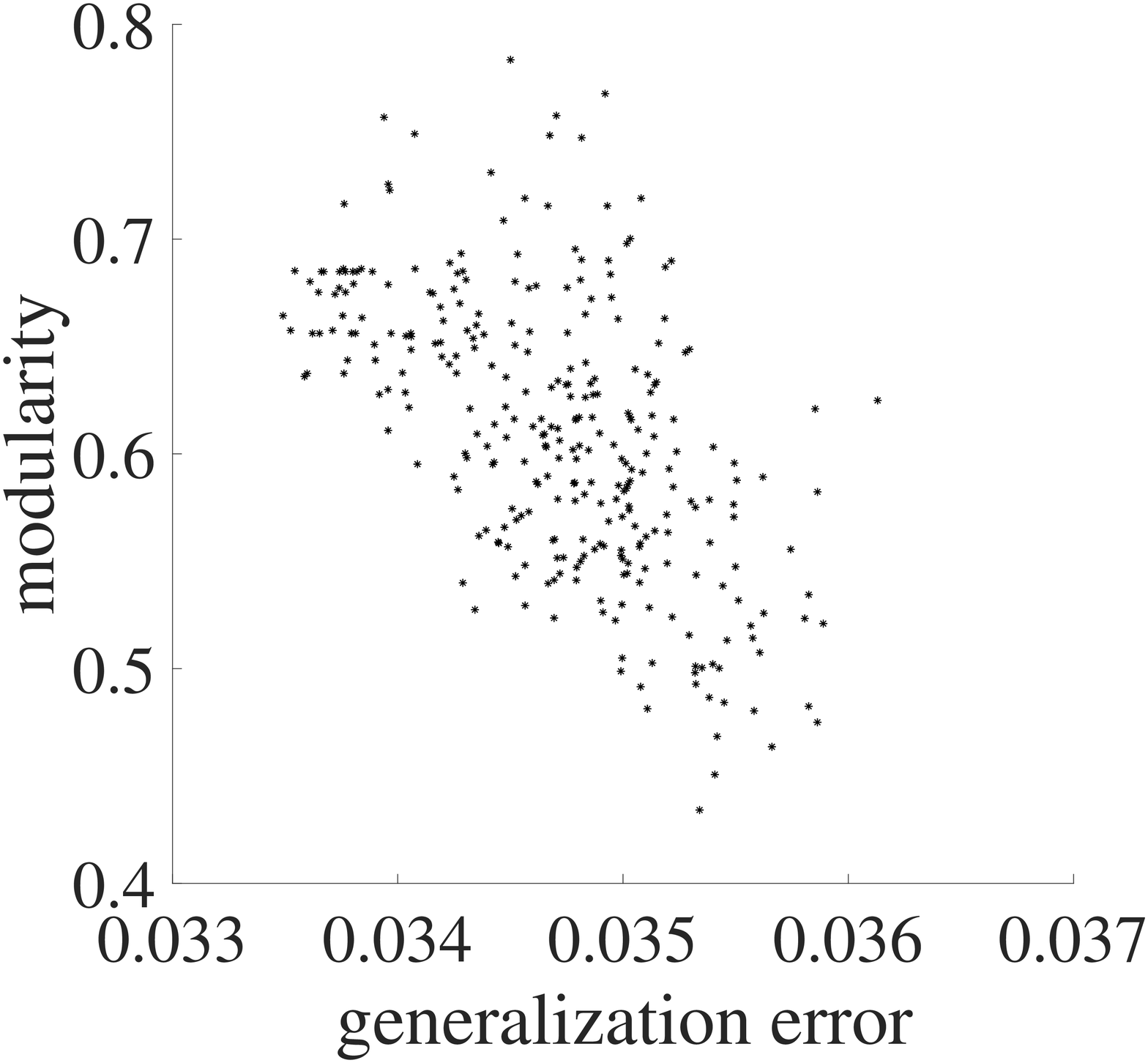}
  \includegraphics[width=35mm]{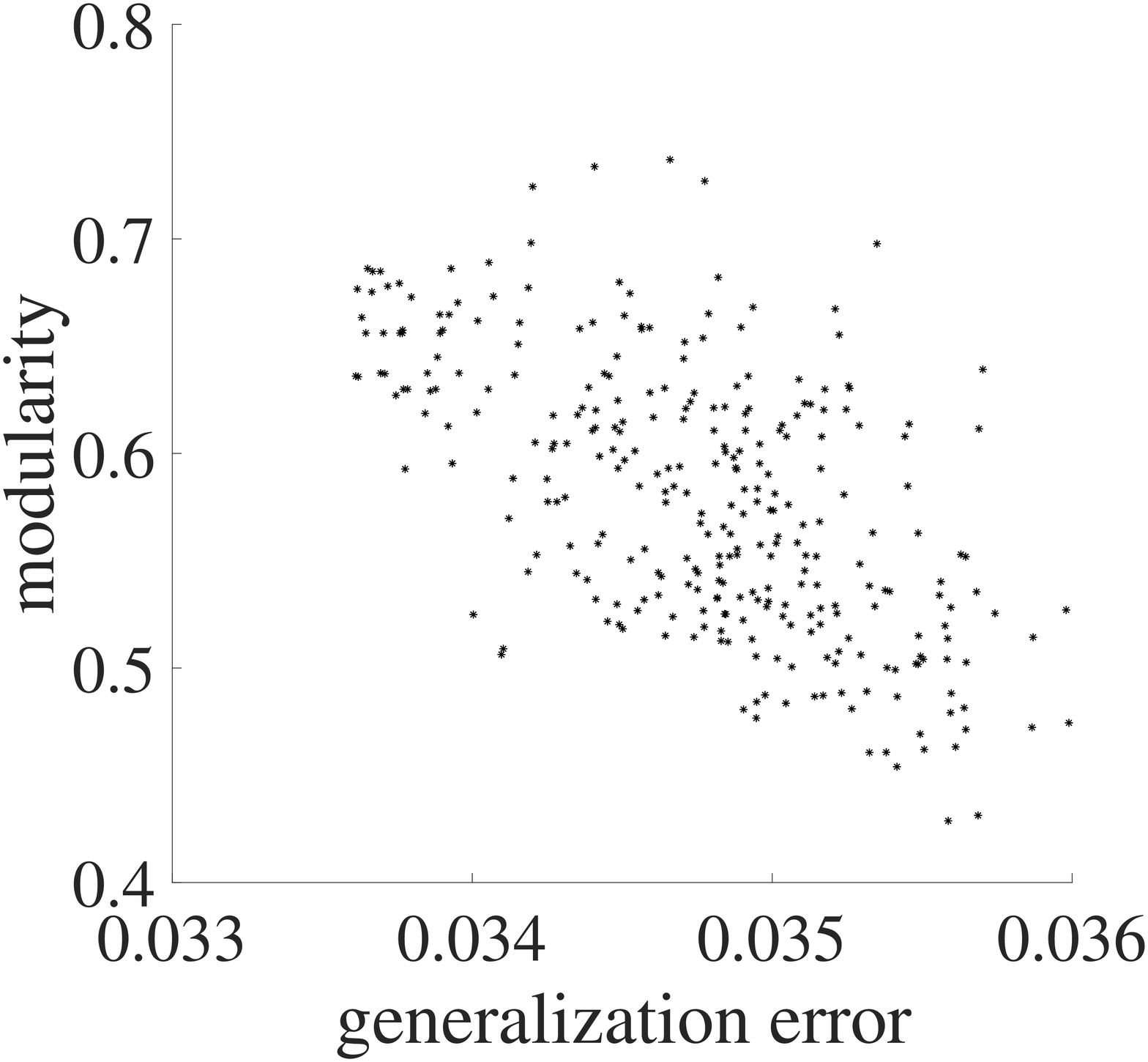}\\
  \includegraphics[width=35mm]{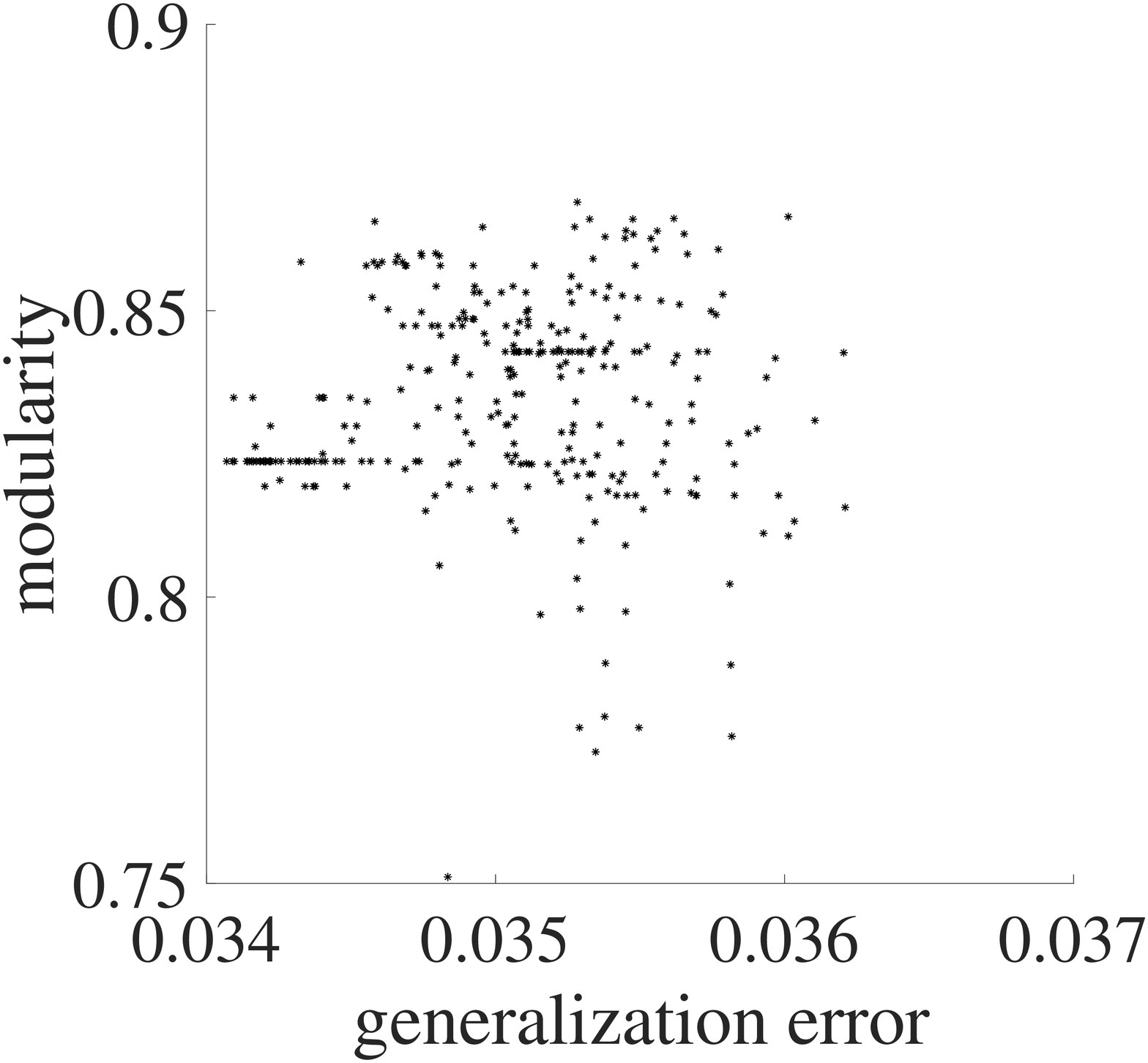}
  \includegraphics[width=35mm]{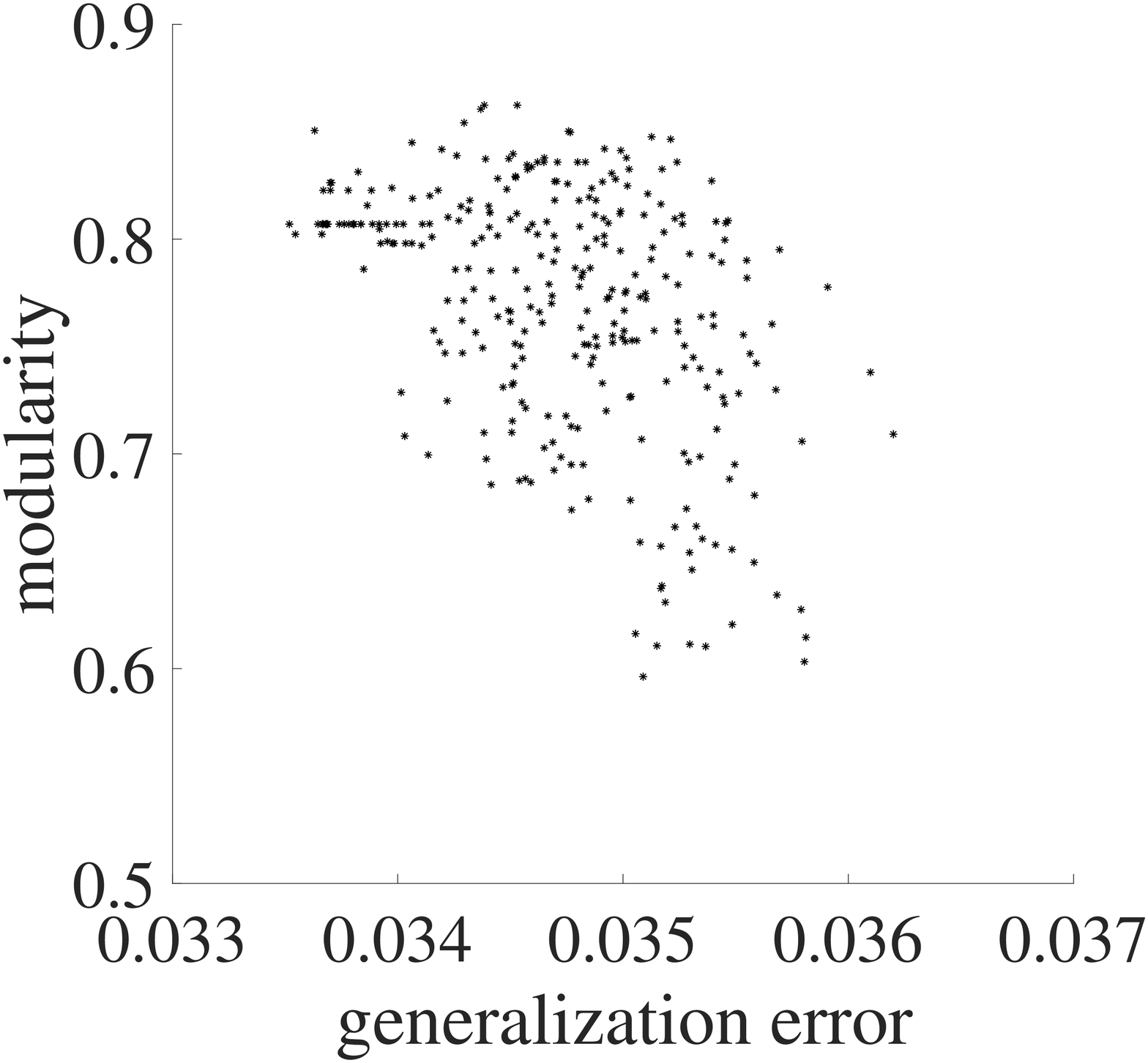}
  \includegraphics[width=35mm]{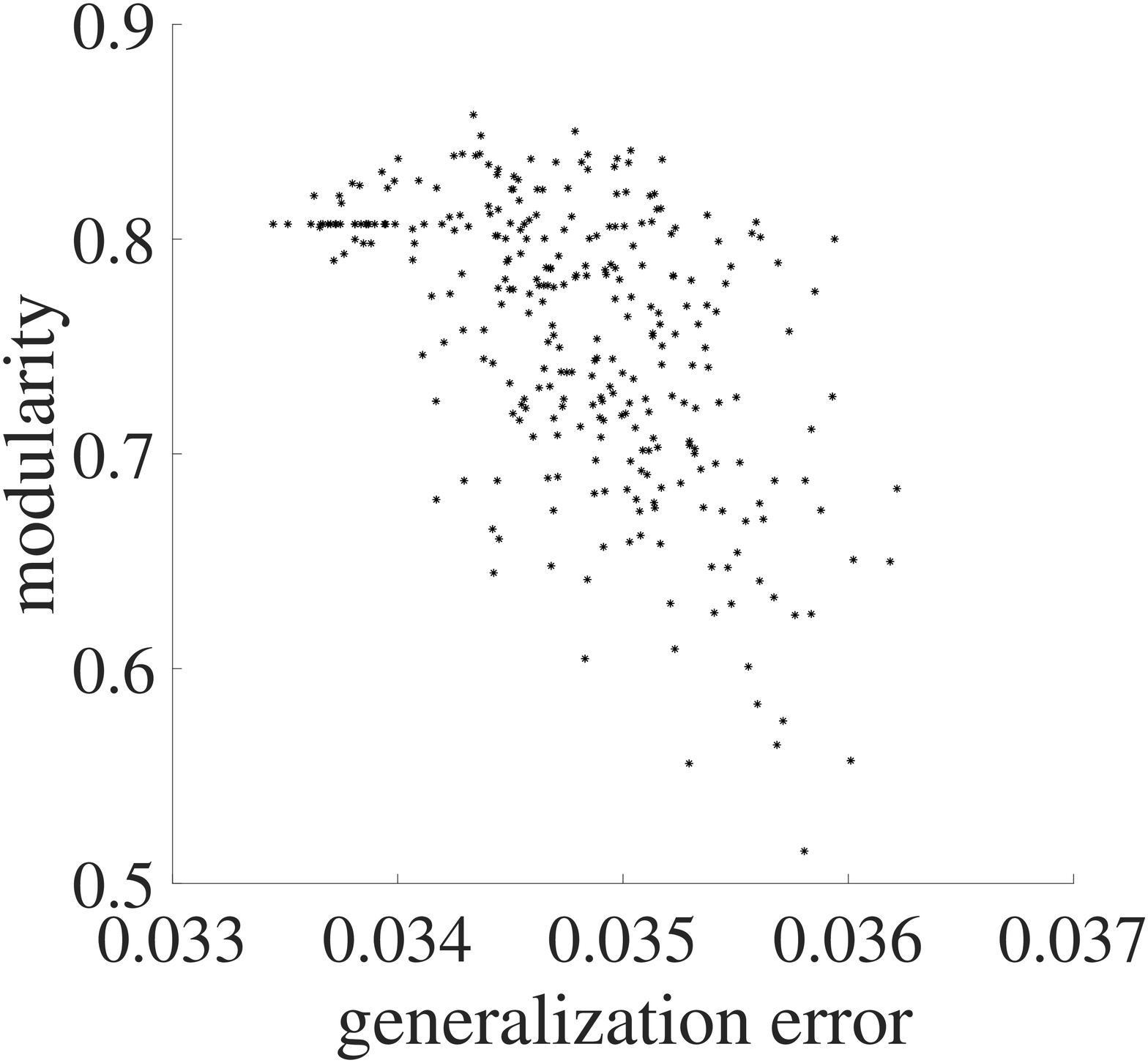}\vspace{-3mm}
  \caption{Relationship between generalization error and modularity. Community detection was performed with layered neural networks trained by multiple independent sets of data. The LASSO hyperparameter $\lambda$ is $1.0\times 10^{-5}$ (left), $1.0\times 10^{-6}$ (center) $1.0\times 10^{-7}$ (right). The weight removing hyperparameter  $\xi$ is $0.1$ (top), $0.3$ (center) $0.6$ (bottom). Better trained results were obtained (with smaller generalization errors) when the trained neural networks had clearer community divisions (with high modularity) except when $(\lambda ,\xi )=(1.0\times 10^{-5}, 0.6)$.}
  \label{fig:exp2}
\end{minipage}\\ \vspace{2mm}

\def\@captype{table}
\begin{minipage}[t]{130mm}
\centering
\scalebox{0.7}{
\begin{tabular}{c|c||c|c|c}\Hline
         \multicolumn{2}{c||}{} & \multicolumn{3}{c}{$\lambda$} \\ \cline{3-5}
         \multicolumn{2}{c||}{} &  \makebox[7em]{$1.0\times 10^{-5}$} & \makebox[7em]{$1.0\times 10^{-6}$} & \makebox[7em]{$1.0\times 10^{-7}$}\\ \hline \hline
         & $0.1$ & $R: -0.24, p: 2.1\times 10^{-5}$ & $R: -0.23, p: 5.0\times 10^{-5}$ & $R: -0.17, p: 3.6\times 10^{-3}$ \\ \cline{2-5}
 $\xi$ & $0.3$ & $R: -0.43, p: 3.5\times 10^{-15}$ & $R: -0.58, p: 1.5\times 10^{-28}$ & $R: -0.61, p: 1.3\times 10^{-31}$ \\ \cline{2-5}
         & $0.6$ & $R: 0.040, p: 0.49$ & $R: -0.44, p: 7.9\times 10^{-16}$ & $R: -0.55, p: 5.1\times 10^{-25}$ \\ \Hline
\end{tabular}
}\vspace{-2mm}
\tblcaption{The correlation $R$ and the p-value $p$ for the generalization errors and the modularities. }
\label{tab:cc}
\end{minipage}\\ \vspace{2mm}

\def\@captype{table}
\begin{minipage}[t]{130mm}
\centering
\scalebox{0.7}{
\begin{tabular}{c||c|c|c|c|c}\Hline
 $\alpha$ & $0$ & $0.1$ & $0.2$ & $0.3$ & $0.4$ \\ \hline \hline
 $R$ & $-0.58$ & $-0.72$ & $-0.58$ & $-0.62$ & $-0.56$ \\ \hline
 $p$ & $1.5\times 10^{-28}$ & $5.4\times 10^{-49}$ & $3.1\times 10^{-28}$ & $4.9\times 10^{-33}$ & $1.5\times 10^{-26}$ \\ \Hline
 $\alpha$ & $0.5$ & $0.6$ & $0.7$ & $0.8$ & $0.9$ \\ \hline \hline
 $R$ & $-0.59$ & $-0.33$ & $-0.32$ & $-0.014$ & $0.14$ \\ \hline
 $p$ & $2.6\times 10^{-29}$ & $6.2\times 10^{-9}$ & $1.5\times 10^{-8}$ & $0.81$ & $0.018$ \\ \Hline
\end{tabular}
}\vspace{-2mm}
\tblcaption{The correlation $R$ and the p-value $p$ for the generalization errors and the modularities with varying dependence between input data. The parameter $\alpha$ represents the strength of dependence.}
\label{tab:coralpha}
\end{minipage}
\end{figure*}

In the experiment, we iterated the neural network training and community detection from the trained network $300$ times, using $15$ dimensional data that are generated in the same way as the experiment described in section \ref{sec:independent}. 
The generalization error and modularity results for nine pairs of hyperparameters $\{\lambda, \xi\}$ are shown in Figure \ref{fig:exp2}, where $\lambda$ is the LASSO hyperparameter and $\xi$ is the weight removing hyperparameter. For the smaller $\lambda$ and $\xi$, the overall modularities were lower, which indicates that there were more connections between mutually independent neural networks. It was experimentally shown for some hyperparameters that better trained results were obtained (with smaller generalization errors) when the trained neural networks had clearer community divisions (with higher modularity). Table \ref{tab:cc} shows the correlations and the p-values for given $\{\lambda, \xi\}$. 

\subsubsection{Correlation between modularity and generalization error when using input data of correlated dimensions}

We also evaluated the relationship between modularity and generalization error, when there is dependence between dimensions of input data. Values of each dimensions of input data were given by
\begin{eqnarray}
  x^n_j = \begin{cases}
    z^n_j & (1 \leq j \leq 5), \\
    (1-\alpha)\times z^n_j +\alpha \times z^n_{j-5} & (6 \leq j \leq 10),\\
    (1-\alpha)\times z^n_j +\alpha \times z^n_{j-10} & (otherwise),
  \end{cases}  
  \label{eq:alpha}
\end{eqnarray}
where $\alpha$ is a control parameter of dependence in input data and $z^n_j\overset{\text\small\rm{i.i.d.}}{\sim}\mathcal{N}(0,3)$. 

We varied the parameter $\alpha$ from $0$ to $0.9$, and checked the correlation between modularity and generalization error for each setting. Here, we set the hyperparameters at $(\lambda, \xi)=(1.0\times 10^{-6},0.3)$, and used the same experimental settings other than the generation method of input data and the hyperparameters. 
Table \ref{tab:coralpha} shows the correlations and the p-values for varying control parameter $\alpha$. 
It was shown that there was a correlation between generalization error and modularity when the dependence between input data was not so strong. Roughly, for the larger $\alpha$, modularity and generalization error have the weaker correlation. This is because the three sets of input data were not necessarily be decomposed into different communities, even if the training result of neural network was appropriate. If the input data contain strongly dependent dimensions, it is necessary to remove such dimensions in advance, for analyzing the extracted modular structure properly. To construct a method for improving input data appropriately based on their dependency is a future work. 

\subsection{Knowledge discovery from modular representation} 
\label{sec:kd}

In order to show that the modular representation extracts the global structure of a trained neural network, we applied the proposed method to a neural network trained with practical data. We used data that represent the characteristics of each municipality in Japan \cite{estat}. The characteristics shown in Table \ref{tab:notationsdata} were used as the input and output data, and the data of municipalities that had any missing value were removed. There were $1905$ data, and we divided them into $952$ training data and $953$ test data. Before the neural network was trained, all the dimensions for all sets of data were converted through the function of $\log (1+x)$, because the original data are highly biased. The results are shown in Figure \ref{fig:exp3data}.

\begin{figure*}
  \centering
  \includegraphics[width=145mm]{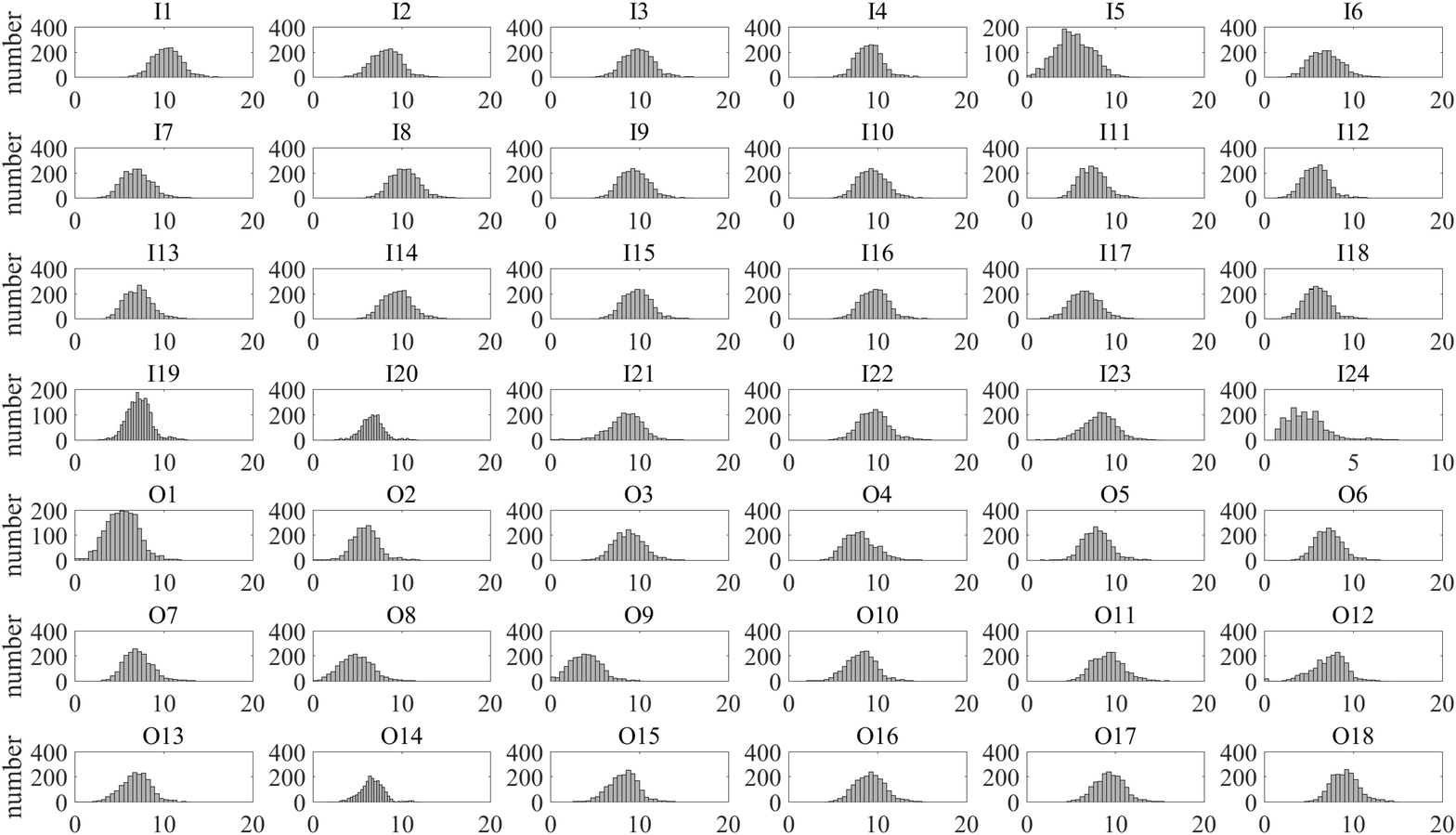}\vspace{-5mm}
  \caption{Histogram of each dimension of data that contain the characteristics of each municipality in Japan. The data were converted through the function of $\log (1+x)$. The notations are shown in Table \ref{tab:notationsdata}.}
  \label{fig:exp3data}
\end{figure*}

We iterated the neural network training and community detection from the trained network $300$ times, using the above data. The trained neural network and the modular representation with minimum generalization error are shown in Figures \ref{fig:exp3o} and \ref{fig:exp3m}, respectively. The correlation $R$ between modularity and generalization error was $-0.028$. Figure \ref{fig:exp3m} shows, for example, that the number of births, deaths, marriages, divorces, people who engage in secondary industry work, and unemployed people (A4) were inferred from the population of transference, the number of out-migrants, households, secondary industry establishments and so on (A1).

\begin{figure*}
  \centering
  \includegraphics[width=135mm]{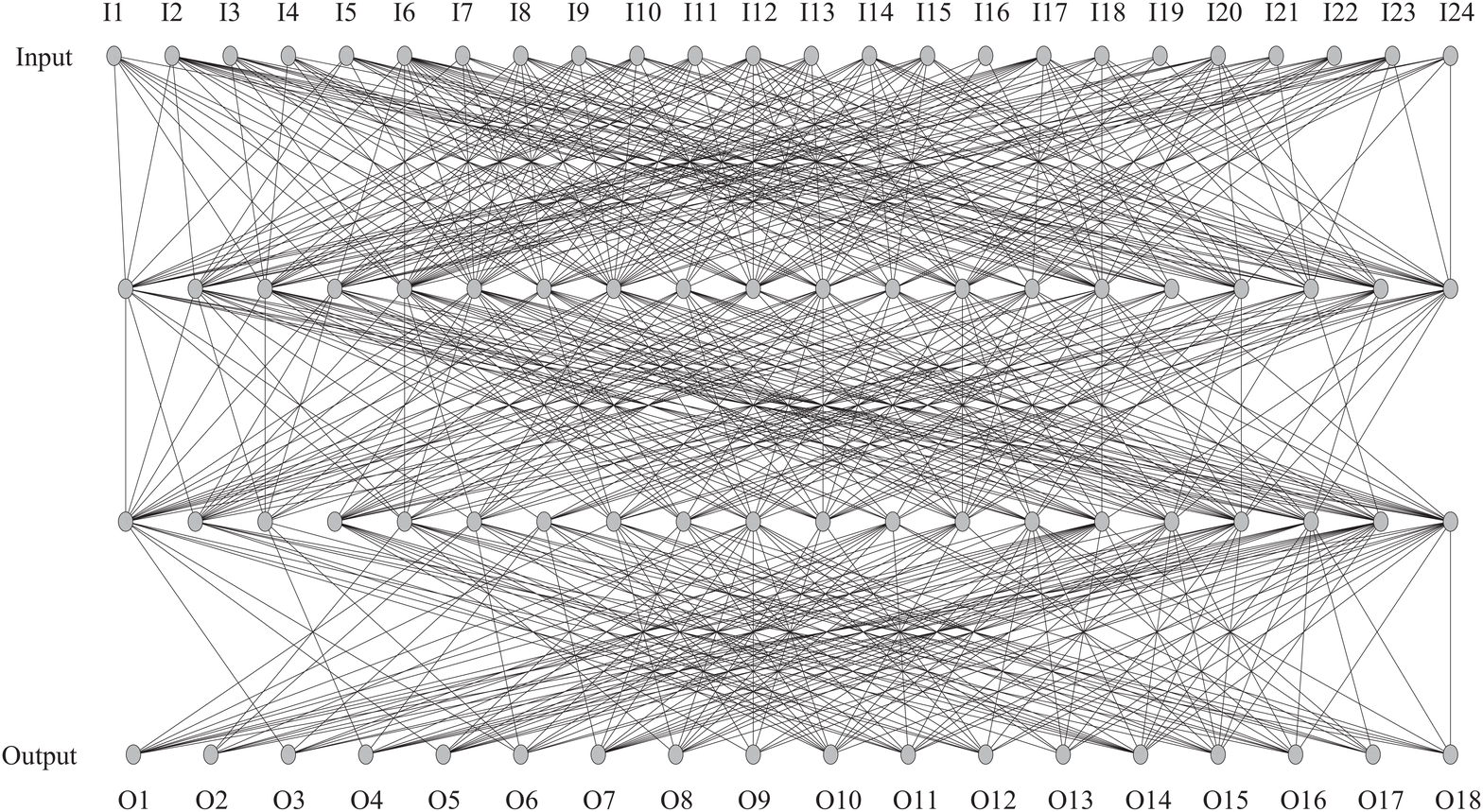}\vspace{-5mm}
  \caption{Trained neural network for practical data. The input and output data notations are shown in Table \ref{tab:notationsdata}.}\vspace{6mm}
  \label{fig:exp3o}
  \centering
  \includegraphics[width=135mm]{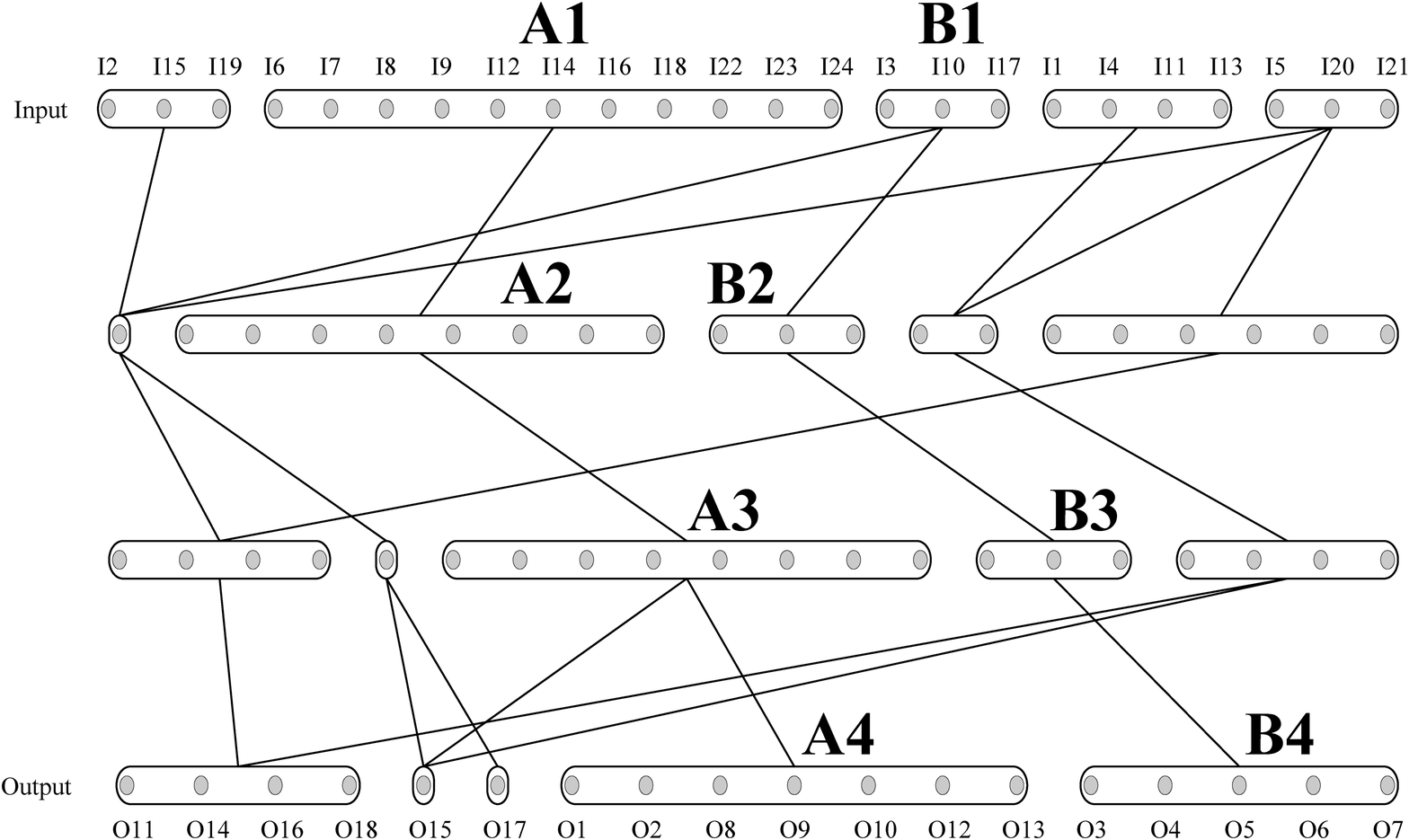}\vspace{-1mm}
  \caption{Extracted modular representation of trained neural network. This figure shows, for example, that the number of births, deaths, marriages, divorces, people who engage in secondary industry work, and unemployed people (A4) were inferred from the population of transference, the number of out-migrants, households, secondary industry establishments and so on (A1).}
  \label{fig:exp3m}
\end{figure*}

From the extracted modular representation, we found not only the grouping of the input and output units, but also the relational structure between the communities of the input, output and hidden layers. For instance, the third community from the right in the depth $2$ layer (B2) and the second community from the right in the depth $3$ layer (B3) only connected to partial input and output units: they were used only for inferring the number of nuclear family households, single households, nuclear family households with members aged 65 and older, and elderly households (B4), from the population between 15 and 64 years of age, the number of general households and executives (B1).

\section{Discussion} 
\label{sec:discussion}

In this section, we discuss the proposed algorithm from four viewpoints, the community detection method, the validation of the extracted result, the scalability of our method, and the application. 

Firstly, to extract a modular representation from a trained neural network, we employed a basic iterative community detection method for each layer. It is possible to modify this method, for example, by using the weights of connections or the connections in further layers. Utilizing the output of each unit might also improve preciseness of the community detection result. 
In general, connection weights of a neural network can be trained the more appropriately with the more training data, resulting in the more valid modular representation extraction. For a given set of data, the optimal hyperparameters $\lambda$ and $\xi$ in the sense of the smallest generalization error can be found by cross-validation, but it takes heavy computational cost. To seek a method for determining the sufficient number of training data, or for optimizing community detection methods and hyperparameters according to the task with small computational complexity is future work.

Secondly, knowledge discovered from a modular representation depends on both the data and the analyst who utilizes the proposed method. For quantitative evaluation of extracted modular structure, statistical hypothesis test or statistical model selection method is required. However, such method has not been constructed in the field of network analysis, so it is also an important mathematical task in the future. Experimentally, there are both sensitive and robust communities which do and do not depend on them. Therefore, it becomes important to separate the essential results from fluctuations. We anticipate that our method will form the basic analytic procedure of such a study. 

Thirdly, in this paper, we experimentally evaluated the relationship between modularity and generalization error. It is well known that a community detection technique can be employed for large size networks. The analysis of larger datasets with higher dimensions would provide further information on layered neural networks. For such large datasets, it would also be important to evaluate the effectiveness of parallel computation, using the independent neural networks extracted with our proposed method. 

And lastly, our proposed community detection method can be used for various applications, such as neural network compression. For instance, it would be possible to use modularity index of a resulting community structure as a penalty term for neural network regularization. 

\section{Conclusion} 
\label{sec:conclusion}

Layered neural networks have achieved a significant improvement in terms of classification or regression accuracy over a wide range of applications by their ability to capture the complex hidden structure between input and output data. However, the discovery or interpretation of knowledge using layered neural networks has been difficult, since its internal representation consists of many nonlinear and complex parameters. 

In this paper, we proposed a new method for extracting a modular representation of a trained layered neural network. The proposed method detects communities of units with similar connection patterns, and determines the relational structure between such communities. We demonstrated the effectiveness of the proposed method experimentally in three applications. (1) It can decompose a layered neural network into a set of small independent networks, which divides the problem and reduces the computation time. (2) The trained result can be estimated by using a modularity index, which measures the effectiveness of a community detection result. And (3) providing the global relational structure of the network would be a clue to discover knowledge from a trained neural network. 

\section*{Appendix} 
\label{sec:appendix}

Table \ref{tab:notationsdata} shows the notations of the data \cite{estat} used in the experiment described in section \ref{sec:kd}. The experimental settings of the parameters are shown in Table \ref{tab:notations}. 

\begin{table}[h]
\caption{Notations of the data.}
  \centering
  \footnotesize
  \begin{tabular}{c|p{5cm}||c|p{5cm}} \Hline
    name & \parbox{5cm}{\strut{}meaning\strut} & name & \parbox{5cm}{meaning\strut} \\ \hline \hline
    I1 & \parbox{5cm}{\strut{}total population\strut} & 
	I22 & \parbox{5cm}{number of employed people by business location\strut} \\ \hline
    I2 & \parbox{5cm}{\strut{}population under 15 years of age\strut} & 
	I23 & \parbox{5cm}{number of commuters from other municipalities\strut} \\ \hline
    I3 & \parbox{5cm}{\strut{}population between 15 and 64 years of age\strut} & 
	I24 & \parbox{5cm}{number of post offices\strut} \\ \hline
    I4 & \parbox{5cm}{\strut{}population aged 65 and older\strut} & 
	O1 & \parbox{5cm}{number of births\strut} \\ \hline
    I5 & \parbox{5cm}{\strut{}foreign population\strut} & 
	O2 & \parbox{5cm}{number of deaths\strut} \\ \hline
    I6 & \parbox{5cm}{\strut{}population of transference\strut} & 
	O3 & \parbox{5cm}{number of nuclear family households\strut} \\ \hline
    I7 & \parbox{5cm}{\strut{}number of out-migrants\strut} & 
	O4 & \parbox{5cm}{number of single households\strut} \\ \hline
    I8 & \parbox{5cm}{\strut{}daytime population\strut} & 
	O5 & \parbox{5cm}{number of nuclear family households with members aged 65 and older\strut} \\ \hline
    I9 & \parbox{5cm}{\strut{}number of households\strut} & 
	O6 & \parbox{5cm}{number of elderly couple households\strut} \\ \hline
    I10 & \parbox{5cm}{\strut{}number of general households\strut} & 
	O7 & \parbox{5cm}{number of elderly single households\strut} \\ \hline
    I11 & \parbox{5cm}{\strut{}number of establishments\strut} & 
	O8 & \parbox{5cm}{number of marriages\strut} \\ \hline
    I12 & \parbox{5cm}{\strut{}number of secondary industry establishments\strut} & 
	O9 & \parbox{5cm}{number of divorces\strut} \\ \hline
    I13 & \parbox{5cm}{\strut{}number of tertiary industry establishments\strut} & 
	O10 & \parbox{5cm}{number of secondary industry workers\strut} \\ \hline
    I14 & \parbox{5cm}{\strut{}number of workers\strut} & 
	O11 & \parbox{5cm}{number of tertiary industry workers\strut} \\ \hline
    I15 & \parbox{5cm}{\strut{}labor force population\strut} & 
	O12 & \parbox{5cm}{number of employees in manufacturing industry\strut} \\ \hline
    I16 & \parbox{5cm}{\strut{}number of employed people\strut} & 
	O13 & \parbox{5cm}{number of unemployed people\strut} \\ \hline
    I17 & \parbox{5cm}{\strut{}number of executives\strut} & 
	O14 & \parbox{5cm}{number of primary industry employees\strut} \\ \hline
    I18 & \parbox{5cm}{\strut{}number of employees with employment\strut} & 
	O15 & \parbox{5cm}{number of secondary industry employees\strut} \\ \hline
    I19 & \parbox{5cm}{\strut{}number of employees without employment\strut} & 
	O16 & \parbox{5cm}{number of tertiary industry employees\strut} \\ \hline
    I20 & \parbox{5cm}{\strut{}number of family workers\strut} & 
	O17 & \parbox{5cm}{number of employees\strut} \\ \hline
    I21 & \parbox{5cm}{\strut{}number of commuters to other municipalities\strut} & 
	O18 & \parbox{5cm}{number of employed workers in their municipalities\strut} \\ \Hline
  \end{tabular}
\label{tab:notationsdata}
\end{table}
\begin{table}[h]
\caption{The experimental settings of the parameters.}
  \centering
  \scalebox{0.85}{
  \begin{tabular}{c|p{6cm}|c|c|c} \Hline
    name & meaning & Exp.1 & Exp.2 & Exp.3 \\ \hline \hline
    $a_1$ & \parbox{6cm}{\strut{}mean iteration number of neural network training per set of data\strut} & $4000$ & \multicolumn{2}{c}{$2000$} \\ \hline
    $n$ & number of training data sets & $3000$ & $500$ & $952$\\ \hline
    $m$ & number of test data sets & $0$ & $500$ & $953$\\ \hline
    $\{l_d\}$ & \parbox{6cm}{\strut{}number of units in depth $d$ layer\strut} & $\{45,45,45\}$ & $\{15,15,15\}$ & $\{24,20,20,18\}$\\ \hline
    $D$ & \parbox{6cm}{\strut{}number of layers including input, hidden, and output layers\strut} & \multicolumn{2}{c|}{$3$} & $4$\\ \hline
    $\lambda$ & hyperparameter of LASSO & $1.0\times 10^{-6}$ & * & $1.0\times 10^{-7}$\\ \hline
    $\epsilon$ & \parbox{6cm}{\strut{}hyperparameter for convergence of neural network\strut} & \multicolumn{3}{c}{$0.001$}\\ \hline
    $\xi$ & weight removing hyperparameter & $0.3$ & * & $0.5$\\ \hline
    $C$ & \parbox{6cm}{\strut{}number of communities per layer\strut} & \multicolumn{2}{c|}{$3$} & $5$\\ \hline
    Method & \parbox{6cm}{\strut{}method for defining bundled connections\strut} & \multicolumn{2}{c|}{$2$} & $3$\\ \hline
    $\zeta$ & \parbox{6cm}{\strut{}threshold for defining bundled connections\strut} & \multicolumn{3}{c}{$0.3$}\\ \hline
    $x_{\mathrm{min}}$ & \parbox{6cm}{\strut{}minimum value of normalized input data\strut} & \multicolumn{2}{c|}{$-3$} & $-1$\\ \hline
    $x_{\mathrm{max}}$ & \parbox{6cm}{\strut{}maximum value of normalized input data\strut} & \multicolumn{2}{c|}{$3$} & $1$\\ \Hline
  \end{tabular}
  }
*: The nine parameters shown in the caption of Figure \ref{fig:exp2} are used.
\label{tab:notations}
\end{table}

\section*{Acknowledgements} 

We would like to thank Akisato Kimura for his helpful comments on this paper.

\end{document}